\documentclass[american]{article}
\usepackage[T1]{fontenc}
\usepackage[utf8]{inputenc}
\usepackage{color}
\usepackage{babel}
\usepackage{amsmath}
\usepackage{amsthm}
\usepackage{amssymb}
\usepackage[authoryear]{natbib}
\usepackage[pdfusetitle,
 bookmarks=true,bookmarksnumbered=false,bookmarksopen=false,
 breaklinks=false,pdfborder={0 0 1},backref=false,colorlinks=true]
 {hyperref}
\hypersetup{
 linkcolor=blue, citecolor=blue}

\makeatletter
\theoremstyle{plain}
\newtheorem{thm}{\protect\theoremname}
\theoremstyle{definition}
\newtheorem{defn}[thm]{\protect\definitionname}
\theoremstyle{plain}
\newtheorem{prop}[thm]{\protect\propositionname}
\theoremstyle{remark}
\newtheorem{rem}[thm]{\protect\remarkname}
\theoremstyle{plain}
\newtheorem{cor}[thm]{\protect\corollaryname}
\theoremstyle{plain}
\newtheorem{lem}[thm]{\protect\lemmaname}
\theoremstyle{definition}
\newtheorem{example}[thm]{\protect\examplename}

\makeatother

\providecommand{\corollaryname}{Corollary}
\providecommand{\definitionname}{Definition}
\providecommand{\examplename}{Example}
\providecommand{\lemmaname}{Lemma}
\providecommand{\propositionname}{Proposition}
\providecommand{\remarkname}{Remark}
\providecommand{\theoremname}{Theorem}

\begin{document}
\title{Explicit integral representations and quantitative bounds for two-layer
ReLU networks}
\author{Anthony Lee\\
School of Mathematics, University of Bristol}
\date{\today}
\maketitle
\begin{abstract}
An approach to construct explicit integral representations for two-layer
ReLU networks is presented, which provides relatively simple representations
for any multivariate polynomial. Quantitative bounds are provided
for a particular, sharpened ReLU integral representation, which involves
a harmonic extension and a projection. The bounds demonstrate that
functions can be approximated with $L^{2}(\mathcal{D})$ errors that
do not depend explicitly on dimension or degree, but rather the coefficients
of their monomial expansions and the distribution $\mathcal{D}$.
We also present a connection to the RKHS of the exponential kernel
$K(x,y)=\exp\left(\left\langle x,y\right\rangle \right)$, and a very
simple integral representation involving additionally multiplication
via a fixed function which has better quantitative bounds.
\end{abstract}
\tableofcontents{}


\subsubsection*{Note on revision}

In this version, some typos and omissions have been corrected. The
only substantive change is the addition of Section~\ref{sec:An-RKHS-perspective}.
More global revisions, e.g. to modify the structure and flow of the
paper in light of this new Section, have been postponed until after
the manuscript has been reviewed. 

\section{Introduction}\label{sec:Introduction}

The capability of trained neural networks to accurately approximate
some high-dimensional functions is now well established empirically,
and has driven spectacular progress in many application areas. A robust
theoretical explanation for this phenomenon, particularly for deep
networks, remains outstanding. In this paper, we focus on an explicit
representation of arbitrary multivariate polynomials in two-layer
ReLU networks, which is relatively simple both to describe and to
bound the performance of. The approach follows the perspective of
two-layer networks adopted in \citet{barron}, i.e. to view a random
two-layer network as an average of independent and identically distributed
random functions $\phi_{1},\ldots,\phi_{n}$, or neurons, that are
pointwise unbiased for the function of interest $f$. The resulting
$\mathcal{O}(n^{-1/2})$ bound on the expected $L^{2}(\mathcal{D})$
error of the function then implies the existence of a deterministic
two-layer network with at most that $L^{2}(\mathcal{D})$ error. As
discussed in \citet{peyre2025mathematics}, the general idea does
not provide bounds on the number of neurons required without making
assumptions on the function to be approximated, e.g. using a Barron
semi-norm that depends on the Fourier transform. The representation
and bounds we present are an attempt to contribute alternative, simple
measures of function difficulty, and we require no assumptions on
the polynomials to obtain quantitative bounds for suitable, sub-Gaussian
and mean $0$, $\mathcal{D}$.

We focus on arbitrary polynomials $f\in\mathcal{P}_{d}$, the space
of polynomials with domain $\mathbb{R}^{d}$. Polynomials have a finite
degree, but we do not need to make explicit the degree in the sequel.
We adopt multi-index notation: for $\alpha\in\mathbb{N}_{0}^{d}$,
$x^{\alpha}=\prod_{i=1}^{d}x_{i}^{\alpha_{i}}$ and $\alpha!=\prod_{i=1}^{d}\alpha_{i}!$.
For a given smooth, non-polynomial, function of interest $\varphi:\mathbb{R}^{d}\to\mathbb{R}$,
we could, e.g., define $f$ via a truncated Taylor expansion of $\varphi$
around $0$:
\[
f(x)=\sum_{\alpha:\left|\alpha\right|\leq m}\frac{\partial^{\alpha}\varphi(0)}{\alpha!}x^{\alpha}=\sum_{k=0}^{m}\sum_{\left|\alpha\right|=k}\frac{\partial^{\alpha}\varphi(0)}{\alpha!}x^{\alpha}=\sum_{k=0}^{m}f_{k}(x),
\]
where $f_{k}\in\mathcal{P}_{d,k}$, the space of homogeneous polynomials
of degree $k$ in $\mathcal{P}_{d}$. The natural way to measure the
size of these polynomials, at least from the perspective of our quantitative
bounds, is via the Fischer norm of its homogeneous components. For
a homogeneous polynomial of degree $k$, $g(x)=\sum_{\left|\alpha\right|=k}c_{\alpha}x^{\alpha}$,
the squared Fischer norm of $g$ is $\left\Vert g\right\Vert _{F}^{2}=\sum_{\left|\alpha\right|=k}c_{\alpha}^{2}\alpha!$.
In view of the Taylor expansion of $\varphi$, this means $\left\Vert f_{k}\right\Vert _{F}^{2}=\sum_{\left|\alpha\right|=k}\frac{\left|\partial^{\alpha}\varphi(0)\right|^{2}}{\alpha!}$.
Of course, one may define a different polynomial approximation to
$\varphi$ if a Taylor approximation is not suitable, or $\varphi$
is not smooth. The restriction to polynomials is not particularly
strong, from the perspective that polynomials are dense in the space
of continuous functions with compact domain and dense in $L^{2}(\mathcal{D})$
for many probability measures $\mathcal{D}$ with suitable tail decay.

In contrast to several other works, we do not begin by expressing
$f$ using a Fourier transform as in \citet{barron}, i.e.,
\[
f(x)=\int_{\mathbb{R}^{d}}\exp\left(i\left\langle \omega,x\right\rangle \right)\tilde{f}(\omega){\rm d}\omega,
\]
where $\tilde{f}$ is the complex-valued Fourier transform of $f$.
This approach appears to be the main starting point for integral representations
\citep[see, e.g.,][]{ji2019neural,hsu2021approximation}. The Fourier
approach, while very general as it can be used for different activation
functions, tends to involve a sequence of approximations, which makes
an explicit and interpretable integral representation difficult to
obtain. Often, there are assumptions required to either represent
the function of interest or provide quantitative bounds. In contrast,
the representation we develop here is specifically for the ReLU network
and we make no assumptions beyond the target function $f$ being a
polynomial: we will develop representations that are totally explicit
for any given polynomial and the quantitative bounds are valid for
any polynomial, accepting of course that they may not always be accurate.
In particular, the bounds are consistent with the notion that a two-layer
network is capable of approximating well functions that do not depend
too much on too many components of their input, and that neither dimension
nor degree are necessarily important.

Our approach involves the explicit solution of an integral equation
involving an appropriate ReLU integral operator from which a ReLU
integral representation is immediately available (Theorem~\ref{thm:relu-rep}).
Once a ReLU integral representation is obtained, it can be optimized
to minimize the expected squared $L^{2}(\mathcal{D})$ error in an
explicit way (Proposition~\ref{prop:network-mse-bound}). As we restrict
our attention to polynomials, the analysis avoids many technicalities,
and we find that there is a ReLU integral representation of a particular
form for every polynomial (Theorem~\ref{thm:general-soln-Tmu}).
It transpires that the most direct application of this result may
produce complicated representations, but that by performing an appropriate
harmonic extension, we are able to recover a representation for an
extension of the function, leading ultimately to a simpler ReLU representation
for the original function (Theorem~\ref{thm:f-relu-rep-integrated}),
which involves sharpened versions of its homogeneous components: these
are defined by running the heat equation backwards in time. Ultimately,
the error bounds (Theorem~\ref{thm:quant-bound-deterministic}) only
involve the Fischer norms of the homogeneous components of $f$.

The network and bounds seem to show some agreement with trained two-layer
ReLU networks for some examples. A careful study would be required
to determine whether these representations are a suitable model for
trained networks. In particular, there are transformations that could
be applied prior to constructing the network, and it is not clear
which of these would minimize the error of the network. In simple
examples, the trained network is consistent with having performed
such a transformation. The bounds also seem to capture the sense that
functions that depend only weakly on each individual input component,
or that depend on only a few input components, are easier to represent.

\section{ReLU networks, integral representations and operator}

\subsection{Two-layer ReLU network}

Let $\varsigma=\max\{x,0\}$, interpreted element-wise if $x$ is
a vector, denote the ReLU activation function. An $L$ layer ReLU
neural network, viewed as a function $f_{{\rm NN}}:\mathbb{R}^{d}\to\mathbb{R}$,
may be written as 
\[
f_{{\rm NN}}=f_{L}\circ\cdots\circ f_{1},
\]
where for $\ell\in\{1,\ldots,L-1\}$, with $W_{\ell}$ a matrix and
$b_{\ell}$ a vector,
\[
f_{\ell}(x)=\varsigma\left(W_{\ell}x+b_{\ell}\right),
\]
 and $f_{L}(x)=a+u^{T}x.$ Such networks, when trained appropriately,
appear to have very impressive representational capacity and amenability
to optimization.

We focus on the case $L=2$, the two-layer ReLU network, which can
be written as
\[
f_{n}(x)=a+\frac{1}{n}\sum_{i=1}^{n}u_{i}\varsigma\left(\left\langle x,v_{i}\right\rangle +b_{i}\right),\qquad x\in\mathbb{R}^{d},
\]
where the parameters of the function are $(a,v_{1:n},b_{1:n},u_{1:n})$
with $a,b,u_{i}\in\mathbb{R}$ and $v_{i}\in\mathbb{R}^{d}$. One
can think of $f_{n}$ as mainly involving the average of parameterized
functions
\[
\phi(\cdot;u,v,b)=u\varsigma\left(\left\langle \cdot,v\right\rangle +b\right),
\]
each of which cannot be an accurate representation of some target
function $f$, but whose combination clearly can. Although considerably
less powerful empirically than $f_{{\rm NN}}$, two-layer networks
are a suitable target for mathematical research, and progress in their
analysis may facilitate progress for deeper networks. 

\subsection{ReLU integral representations}

Since \citet{barron}, one useful perspective is to view the two-layer
network $f_{n}$ as an approximation of an integral representation
of $f$. We define here a suitable ReLU integral representation in
this spirit.
\begin{defn}[ReLU integral representation]
\label{def:ReLU-integral-rep}A ReLU integral representation $(\nu,a,b,u)$
of a function $f\in\mathcal{P}_{d}$ is given by a measure $\nu$,
a constant $a$ and functions $b$ and $u$ such that
\[
f(x)=a+\int_{\mathsf{Z}}u(v,s)\varsigma\left(\left\langle x,v\right\rangle +b(v,s)\right)\nu({\rm d}v,{\rm d}s),\qquad x\in\mathbb{R}^{d}.
\]
\end{defn}

Given one representation, there exist a family of equivalent representations,
by a change of measure. That is, for any probability measure $\pi$
such that $\pi\gg\nu$, one has that $(\pi,a,b,u_{\pi})$ is a ReLU
integral representation for a suitable $u_{\pi}$, see Corollary~\ref{cor:change-measure}.
A corresponding random two-layer network is the function
\[
f_{n}(x)=a+\frac{1}{n}\sum_{i=1}^{n}u_{\pi}(V_{i},S_{i})\varsigma\left(\left\langle x,V_{i}\right\rangle +b(V_{i},S_{i})\right),\quad x\in\mathbb{R}^{d},
\]
where $(V_{i},S_{i})\overset{}{\sim}\pi$ are independent. An upper
bound on the squared $L^{2}(\mathcal{D})$ error of a non-random,
perfectly trained network is the expected squared error associated
with the minimizing choice of $\pi$.
\begin{prop}
\label{prop:network-mse-bound}Assume $f$ has ReLU integral representation
$(\nu,a,b,u)$. Let 
\[
\phi(x;v,s)=u(v,s)\varsigma\left(\left\langle x,v\right\rangle +b(v,s)\right).
\]
\begin{enumerate}
\item For any $\pi\gg\nu$, the random function
\[
f_{n}(\cdot)=a+\frac{1}{n}\sum_{i=1}^{n}\frac{{\rm d}\nu}{{\rm d}\pi}(V_{i},S_{i})\phi\left(\cdot;V_{i},S_{i}\right),
\]
with independent $(V_{i},S_{i})\sim\pi$ satisfies $\mathbb{E}\left[f_{n}(x)\right]=f(x)$.
\item The choice of $\pi$ minimizing $\mathbb{E}_{\pi}\left[\left\Vert f_{n}-f\right\Vert _{L^{2}(\mathcal{D})}^{2}\right]$
satisfies
\[
\mathbb{E}_{\pi}\left[\left\Vert f_{n}-f\right\Vert _{L^{2}(\mathcal{D})}^{2}\right]=\frac{1}{n}\left\{ \left[\int_{\mathsf{Z}}\nu({\rm d}z)\left\Vert \phi(\cdot;z)\right\Vert _{L^{2}(\mathcal{D})}\right]^{2}-\left\Vert f-a\right\Vert _{L^{2}(\mathcal{D})}^{2}\right\} .
\]
In particular, there exists some realization $f_{n}$ such that
\[
\left\Vert f_{n}-f\right\Vert _{L^{2}(\mathcal{D})}^{2}\leq\frac{1}{n}\left\{ \left[\int_{\mathsf{Z}}\nu({\rm d}z)\left\Vert \phi(\cdot;z)\right\Vert _{L^{2}(\mathcal{D})}\right]^{2}-\left\Vert f-a\right\Vert _{L^{2}(\mathcal{D})}^{2}\right\} .
\]
\end{enumerate}
\end{prop}

\begin{proof}
Part 1 follows from Corollary~\ref{cor:change-measure}, while Part
2 follows by applying Lemma~\ref{lem:optimal-random-fcn-dist} in
Appendix~\ref{sec:Optimal-squared-} while observing $f_{n}-f=f_{n}-a-(f-a)$.
\end{proof}

\subsection{An integral operator and equation}

It is clear that a ReLU integral representation can provide a route
to understanding and quantitatively bounding the representational
capacity of a two-layer ReLU network. 
\begin{defn}[ReLU operator, $T_{\mu}$]
For a distribution $\mu$ with finite moments of all orders, i.e.
$\int_{\mathbb{R}^{d}}\left\Vert v\right\Vert ^{2k}\mu({\rm d}v)\in\mathbb{R}$
for all $k\geq0$, we define the integral operator $T_{\mu}:\mathcal{P}_{d}\to\mathcal{P}_{d}$
by
\[
T_{\mu}g(x)=\int_{\mathbb{R}^{d}}g\left(\left\langle x,v\right\rangle v\right)\mu({\rm d}v),\qquad x\in\mathbb{R}^{d}.
\]
\end{defn}

The requirement that $\mu$ have finite moments of all orders, is
necessary and sufficient to ensure that the operator is well-defined
for all $g\in\mathcal{P}_{d}$. We will see in the sequel that $T_{\mu}$
is bijective when $\mu$ is additionally spherically symmetric, i.e.
there is a unique solution $g=T_{\mu}^{-1}f$ for any polynomial $f$.
\begin{defn}[ReLU representative]
A ReLU representative for a function $f$ and distribution $\mu$
is any solution $g$ to $f=T_{\mu}g$.
\end{defn}

In constructing a ReLU integral representation of a function $f$,
we will find that an appropriate route is to identify a ReLU representative
of $f$ for some spherically symmetric $\mu$. Then, we may obtain
the ReLU integral representation via a first-order Taylor expansion.
For a function $g:\mathbb{R}^{d}\to\mathbb{R}$ and a vector $v\in\mathbb{R}^{d}$
we denote by $g_{v}:\mathbb{R}\to\mathbb{R}$ the function $t\mapsto g(tv)$.
We define $\varrho=\delta_{0}+{\rm Leb}(\mathbb{R}_{+})$, the measure
corresponding to the sum of a Dirac measure at $0$ and the Lebesgue
measure on the positive reals.
\begin{thm}
\label{thm:relu-rep}Let $\mu$ be a spherically symmetric distribution,
and assume $f=T_{\mu}g$ for some twice differentiable function $g$.
Then $(\mu\otimes\varrho,a,b,u)$ is a ReLU integral representation
of $f$, where $a=f(0)$, $b(v,s)=-s$ and
\[
u(v,s)=\begin{cases}
2g_{v}'(0) & s=0,\\
2g_{v}''(s) & s>0.
\end{cases}
\]
\end{thm}

\begin{proof}
Fix any $x\in\mathbb{R}^{d}$. Then, 
\[
f(x)=\int_{\mathbb{R}^{d}}g\left(\left\langle x,v\right\rangle v\right)\mu({\rm d}v)=2\int_{\mathbb{R}^{d}}g\left(\left\langle x,v\right\rangle v\right){\bf 1}_{\left\langle x,v\right\rangle >0}\mu({\rm d}v),
\]
since $\mu$ is spherically symmetric and $\left\langle x,v\right\rangle v=\left\langle x,-v\right\rangle (-v)$
for any $v\in\mathbb{R}^{d}$. Let $\tau=\left\langle x,v\right\rangle >0$
and observe that by Taylor's Theorem with integral remainder,
\begin{align*}
g\left(\tau v\right) & =g_{v}(0)+\tau g_{v}'(0)+\int_{0}^{\tau}(\tau-s)g_{v}''(s){\rm d}s\\
 & =g_{v}(0)+\tau g_{v}'(0)+\int_{0}^{\infty}\varsigma(\tau-s)g_{v}''(s){\rm d}s,
\end{align*}
from which we may conclude, noting that $g_{v}(0)=g(0)=f(0)$.
\end{proof}
The use of $g_{v}$ rather than $g$ is to simplify the presentation.
One can also express $u$ using the gradient and Hessian of $g$ by
the identities $g_{v}'(0)=\left\langle v,\nabla g(0)\right\rangle $
and $g_{v}''(s)=\left\langle \nabla^{2}g(sv)v,v\right\rangle $.
\begin{rem}
\label{rem:1d-telgarsky-ack}For $d=1$, $T_{\mu}={\rm Id}$ for $\mu={\rm Uniform}(\{-1,1\})$.
Indeed, then $\left\langle x,v\right\rangle v=xv^{2}=x$, so $T_{\mu}g=g$.
This observation for a different, indicator-based integral representation
rather than ReLU-based integral representation for one-dimensional
functions in \citet[Proposition~2.8]{telgarsky2023deep} inspired
the use of the present expansion to recover the ReLU function naturally
and then commence the search for ReLU representatives.
\end{rem}

The Radon--Nikodym Theorem then provides an immediate corollary for
suitable alternatives to $\mu\otimes\varrho$. These ReLU integral
representations, however, do not correspond to solving the ReLU integral
equation for a different probability measure. Instead, they perturb
$\mu\otimes\varrho$ and modify $u$ accordingly to potentially improve
the performance of the network while retaining the same essential
representation. The distribution of $v$ under $\pi$ is not necessarily
spherically symmetric, even if $\mu$ is.
\begin{cor}
\label{cor:change-measure}Let $(\nu,a,b,u)$ be a ReLU integral representation
of $f$. Let $\pi$ be a probability measure on $\mathbb{R}^{d}\times\mathbb{R}_{+}$
such that $\pi\gg\nu$. Then $(\pi,a,b,u_{\pi})$ is a ReLU integral
representation of $f$ with
\[
u_{\pi}(v,s)=u(v,s)\frac{{\rm d}\nu}{{\rm d}\pi}(v,s).
\]
\end{cor}

\begin{rem}
\label{rem:sphere-invariance-rep}The function $\phi(x;v,s)=u(v,s)\varsigma\left(\left\langle x,v\right\rangle +b(v,s)\right)$
is invariant to rescaling in the sense that
\[
u(v,s)\varsigma\left(\left\langle x,v\right\rangle +b(v,s)\right)=cu(v,s)\varsigma\left(c^{-1}\left\langle x,v\right\rangle +c^{-1}b(v,s)\right),
\]
for any $c>0$. Hence, ReLU integral representations have similar
invariance properties. If $(\nu,a,b,u)$ is a ReLU integral representation
of $f$, and $b$ is homogeneous of degree $1$, which will be the
case in the sequel, we can remove this type of invariance by writing
\[
f(x)=a+\int_{\mathsf{Z}}\left\Vert v\right\Vert u(v,s)\varsigma\left(\left\langle x,\frac{v}{\left\Vert v\right\Vert }\right\rangle +b\left(\frac{v}{\left\Vert v\right\Vert },\frac{s}{\left\Vert v\right\Vert }\right)\right)\nu({\rm d}v,{\rm d}s),\qquad x\in\mathbb{R}^{d}.
\]
We do not pursue this any further here, but it could be useful as
a means of comparing integral representations to trained networks
that have been similarly normalized.
\end{rem}

\section{Solving the integral equation for general polynomials}\label{sec:Solving-the-general}

\subsection{The solution}

The main result of this section is Theorem~\ref{thm:general-soln-Tmu}.
Essentially, for spherically symmetric $\mu$, we construct a basis
for $\mathcal{P}_{d}$ using eigenfunctions of $T_{\mu}$ with eigenvalues
that are strictly positive. This then implies that $T_{\mu}$ is a
bijection and the function $T_{\mu}^{-1}f$ can be obtained by diagonalization.
The eigenvalues depend on dimension, degree, the extent to which the
eigenfunctions are harmonic or radial, and the moments of $\mu$.
We might think of this as a polynomial analogue of the Fourier transform.
\begin{thm}
\label{thm:general-soln-Tmu}Let $\mu$ be a spherically symmetric
distribution. Then $T_{\mu}:\mathcal{P}_{d}\to\mathcal{P}_{d}$ is
an invertible operator. In particular, if $d\geq2$ and $f\in\mathcal{P}_{d}$,
then 
\[
T_{\mu}^{-1}f=\sum_{i\geq0}\sum_{j=0}^{\left\lfloor i/2\right\rfloor }\lambda_{d,i-2j,j,\mu}^{-1}f_{ij},
\]
where $f=\sum_{i\geq0}\sum_{j=0}^{\left\lfloor i/2\right\rfloor }f_{ij}$
is the unique harmonic decomposition of $f$ (see Corollary~\ref{cor:harmonic-decomp-inhomo})
and
\[
\lambda_{d,k,i,\mu}=\frac{\Gamma(\frac{d}{2})}{2^{2i+k}i!}\cdot\frac{(k+2i)!}{\Gamma(k+i+\frac{d}{2})}\int_{\mathbb{R}^{d}}\left\Vert v\right\Vert ^{2(k+2i)}\mu({\rm d}v).
\]
For $d=1$, we have $T_{\mu}^{-1}f=\sum_{k\geq0}\lambda_{1,k,\mu}^{-1}f_{k}$,
where $f=\sum_{k\geq0}f_{k}$ is the unique decomposition of $f$
into monomials, and $\lambda_{1,k,\mu}=\int_{\mathbb{R}}v^{2k}\mu({\rm d}v)$.
\end{thm}

\begin{proof}
Consider $d\geq2$ first. Applying Proposition~\ref{prop:harmonic-radial-mvp-sph-sym}
to the claimed inverse,
\[
T_{\mu}\sum_{i\geq0}\sum_{j=0}^{\left\lfloor i/2\right\rfloor }\lambda_{d,i-2j,j,\mu}^{-1}f_{ij}=\sum_{i\geq0}\sum_{j=0}^{\left\lfloor i/2\right\rfloor }\lambda_{d,i-2j,j,\mu}^{-1}T_{\mu}f_{ij}=\sum_{i\geq0}\sum_{j=0}^{\left\lfloor i/2\right\rfloor }f_{ij}=f.
\]
Hence, for any $f\in\mathcal{P}_{d}$ there exists $g\in\mathcal{P}_{d}$
such that $T_{\mu}g=f$. Now we show that $T_{\mu}$ is injective,
for which it is sufficient to show that $g\neq0$ implies $T_{\mu}g\neq0$,
since $T_{\mu}$ is a linear operator. Let $g\neq0$ be an arbitrary
polynomial. Then we can decompose $g=\sum_{i\geq0}g_{i}$ and $g_{i}=\sum_{j=0}^{\left\lfloor i/2\right\rfloor }g_{ij}$
uniquely by Corollary~\ref{cor:harmonic-decomp-inhomo}, i.e. the
$g_{ij}$ are linearly independent. Then,
\[
T_{\mu}g=\sum_{i\geq0}\sum_{j=0}^{\left\lfloor i/2\right\rfloor }\lambda_{d,i-2j,j,\mu}g_{ij},
\]
and since all the eigenvalues $\lambda_{d,i-2j,j,\mu}$ are positive,
$g\neq0$ implies $T_{\mu}g\neq0$. We conclude that $T_{\mu}$ is
bijective and that its inverse is as claimed. The case $d=1$ follows
essentially the same argument and is omitted.
\end{proof}
\begin{rem}
\label{rem:gaussian-C}If $\mu=N(0,\sigma^{2}I_{d})$ then $V\sim\mu$
has $\frac{1}{\sigma^{2}}\left\Vert V\right\Vert ^{2}\sim\chi^{2}(d)$.
Hence, $\int_{\mathbb{R}^{d}}\left\Vert v\right\Vert ^{2m}\mu({\rm d}v)$
is the product of the $m$th moment of a $\chi^{2}(d)$ random variable
and $\sigma^{2m}$, i.e.
\[
\int_{\mathbb{R}^{d}}\left\Vert v\right\Vert ^{2m}\mu({\rm d}v)=\sigma^{2m}2^{m}\frac{\Gamma(m+\frac{d}{2})}{\Gamma(\frac{d}{2})}.
\]
It follows that when $d\geq2$,
\[
\lambda_{d,k,i,\mu}=\sigma^{2(k+2i)}\frac{(k+2i)!}{i!}\cdot\frac{\Gamma(k+2i+\frac{d}{2})}{\Gamma(k+i+\frac{d}{2})},
\]
and the eigenvalue for a homogeneous harmonic function of degree $k$
is
\[
\lambda_{d,k,0,\mu}=\sigma^{2k}\cdot k!.
\]
For $d=1$, we see that if $\mu=N(0,\sigma^{2})$ then
\[
\lambda_{1,k,\mu}=\sigma^{2k}2^{k}\frac{\Gamma(k+\frac{1}{2})}{\Gamma(\frac{1}{2})}=\sigma^{2k}\frac{(2k)!}{2^{k}k!},
\]
which has dependence on $k$, whereas $\lambda_{1,k,{\rm Uniform}\left(\left\{ -1,1\right\} \right)}=1$.
We notice that the radial distribution affects the homogeneous components
of $g$ differently, which for large $d$ can have a stabilizing effect
on the eigenvalues.
\end{rem}

\subsection{Harmonic decomposition of polynomials}

We denote by $\mathcal{H}_{d,k}$ the set of homogeneous harmonic
polynomials of degree $k$, and by $\mathcal{H}_{d,k,i}$ the set
of functions $\varphi$ such that $\varphi(x)=\left\Vert x\right\Vert ^{2i}h_{k}(x)$
with $h_{k}\in\mathcal{H}_{d,k}$. Lemma~\ref{lem:harmonic-decomp}
is classical, see, e.g. \citet[Theorem~5.7]{axler2013harmonic}. The
construction is explicit: one applies the Laplacian several times
to construct the decomposition. However, we will not require explicit
representation of these functions in the sequel. Corollary~\ref{cor:harmonic-decomp-inhomo}
is also classical, we state it here mainly to provide an explicit
representation of $f$. Although Corollary~\ref{cor:harmonic-decomp-inhomo}
does make sense when $d=1$, it is a little awkward and we will opt
instead to use the decomposition of $f$ into monomials when $d=1$.
\begin{lem}[Harmonic decomposition]
\label{lem:harmonic-decomp}If $f\in\mathcal{P}_{d,k}$, i.e. $f$
is a homogeneous polynomial of degree $k$, then we may decompose
$f$ uniquely as
\[
f(x)=h_{k}(x)+\left\Vert x\right\Vert ^{2}h_{k-2}(x)+\cdots+\left\Vert x\right\Vert ^{2j}h_{k-2j}(x),\qquad x\in\mathbb{R}^{d},
\]
where $j=\left\lfloor k/2\right\rfloor $, and each $h_{r}\in\mathcal{H}_{d,r}$.
\end{lem}

\begin{cor}
\label{cor:harmonic-decomp-inhomo}If $f\in\mathcal{P}_{d}$, then
we may uniquely decompose
\[
f=\sum_{i\geq0}\sum_{j=0}^{\left\lfloor i/2\right\rfloor }f_{ij},
\]
where $f_{ij}\in\mathcal{H}_{d,i-2j,j}$ for each $i$ and $j$.
\end{cor}

\begin{proof}
Decompose $f=\sum_{i\geq0}f_{i}$ where $f_{i}$ is homogeneous of
degree $i$. This decomposition is unique as homogeneous polynomials
of different degrees are linearly independent. Then apply Lemma~\ref{lem:harmonic-decomp}
to each homogeneous polynomial, i.e. $f_{i}=\sum_{j}f_{ij}$ where
$f_{ij}\in\mathcal{H}_{d,i-2j,j}$, which is also unique.
\end{proof}

\subsection{Spectral analysis of $T_{\mu}$}\label{subsec:Spectral-analysis-of}

The proof of Lemma~\ref{lem:explicit-lambda-k} is in Appendix~\ref{sec:Funk=002013Hecke-formulas},
and relies on a general Funk--Hecke formula. The result identifies
eigenvalue-eigenfunction pairs for the operator $S_{k,i}$ defined
by 
\[
S_{k,i}g(u)=\int_{S^{d-1}}\left\langle u,v\right\rangle ^{k+2i}g(v)\sigma({\rm d}v),\qquad u\in S^{d-1},
\]
where $\sigma={\rm Uniform}(S^{d-1})$, the uniform distribution on
the sphere $S^{d-1}$. In particular, functions in $\mathcal{H}_{d,k}$
are eigenfunctions of $S$.
\begin{lem}
\label{lem:explicit-lambda-k}Let $d\geq2$, $k\geq0$, $i\geq0$,
and $h_{k}\in\mathcal{H}_{d,k}$. Then $S_{k,i}g=\lambda_{d,k,i}h_{k}$,
with
\[
\lambda_{d,k,i}=\frac{\Gamma(\frac{d}{2})}{2^{2i+k}i!}\cdot\frac{(k+2i)!}{\Gamma(k+i+\frac{d}{2})}.
\]
\end{lem}

Lemma~\ref{lem:explicit-lambda-k} can be used to identify eigenvalue-eigenfunction
pairs for the ReLU operator $T_{\mu}$ when $\mu={\rm Uniform}(S^{d-1})$
. In particular, functions in $\mathcal{H}_{d,k,i}$ are eigenfunctions
of $T_{\mu}$ for this $\mu$.
\begin{lem}
\label{lem:harmonic-radial-mvp}Let $d\geq2$, $k\geq0$, $i\geq0$,
$g\in\mathcal{H}_{d,k,i}$, and $\mu={\rm Uniform}(S^{d-1})$. Then
$T_{\mu}g=\lambda_{d,k,i}g$, where $\lambda_{d,k,i}$ is as in Lemma~\ref{lem:explicit-lambda-k}.
\end{lem}

\begin{proof}
We can write $g(x)=\left\Vert x\right\Vert ^{2i}h_{k}(x)$ for some
$h_{k}\in\mathcal{H}_{d,k}$. Let $x\in\mathbb{R}^{d}$ and $u=x/\left\Vert x\right\Vert $.
Using homogeneity and Lemma~\ref{lem:explicit-lambda-k},
\begin{align*}
T_{\mu}g(x) & =\int_{S^{d-1}}g\left(\left\langle x,v\right\rangle v\right)\mu({\rm d}v)\\
 & =\int_{S^{d-1}}\left\Vert \left\langle x,v\right\rangle v\right\Vert ^{2i}h_{k}(\left\langle x,v\right\rangle v)\mu({\rm d}v)\\
 & =\left\Vert x\right\Vert ^{2i+k}\int_{S^{d-1}}\left\Vert \left\langle u,v\right\rangle v\right\Vert ^{2i}h_{k}(\left\langle u,v\right\rangle v)\mu({\rm d}v)\\
 & =\left\Vert x\right\Vert ^{2i+k}\int_{S^{d-1}}\left\langle u,v\right\rangle ^{2i+k}h_{k}(v)\mu({\rm d}v)\\
 & =\left\Vert x\right\Vert ^{2i+k}\lambda_{d,k,i}h_{k}(u)\\
 & =\lambda_{d,k,i}\left\Vert x\right\Vert ^{2i}h_{k}(x),
\end{align*}
from which we conclude.
\end{proof}
We now proceed to generalize Lemma~\ref{lem:harmonic-radial-mvp}
to spherically symmetric distributions $\mu$. Recall that a spherically
symmetric random vector taking values in $\mathbb{R}^{d}$ is necessarily
of the form
\[
V=RU,
\]
where $U\sim{\rm Uniform}(S^{d-1})$ and $R$ is an independent radial
random variable. In particular, functions in $\mathcal{H}_{d,k,i}$
are also eigenfunctions of the operator $T_{\mu}$, and the eigenvalues
are scaled according to the radial distribution.
\begin{prop}
\label{prop:harmonic-radial-mvp-sph-sym}Let be $\mu$ a spherically
symmetric distribution on $\mathbb{R}^{d}$. Then
\begin{enumerate}
\item For $d=1$, $g\in\mathcal{P}_{1,k}$, we have $T_{\mu}g=\lambda_{1,k,\mu}g$,
where $\lambda_{1,k,\mu}=\int_{\mathbb{R}}v^{2k}\mu({\rm d}v)$.
\item For $d\geq2$, $g\in\mathcal{H}_{d,k,i}$, we have $T_{\mu}g=\lambda_{d,k,i,\mu}g$
where 
\[
\lambda_{d,k,i,\mu}=\int_{\mathbb{R}^{d}}\left\Vert v\right\Vert ^{2(k+2i)}\mu({\rm d}v)\lambda_{d,k,i},
\]
with $\lambda_{d,k,i}$ from Lemma~\ref{lem:harmonic-radial-mvp}.
\end{enumerate}
\end{prop}

\begin{proof}
Consider $d\geq2$. Let $\sigma$ be the uniform distribution on the
sphere $S^{d-1}$. For $v\in\mathbb{R}^{d}$, write $\hat{v}=v/\left\Vert v\right\Vert $.
Let $x\in\mathbb{R}^{d}$. For $d\geq2$, since $g$ is homogeneous
of degree $k+2i$, Lemma~\ref{lem:harmonic-radial-mvp} gives
\begin{align*}
T_{\mu}g(x) & =\int_{\mathbb{R}^{d}}g\left(\left\langle x,v\right\rangle v\right)\mu({\rm d}v)\\
 & =\int_{\mathbb{R}^{d}}g\left(\left\langle x,\left\Vert v\right\Vert \hat{v}\right\rangle \left\Vert v\right\Vert \hat{v}\right)\mu({\rm d}v)\\
 & =\int_{\mathbb{R}^{d}}\left\Vert v\right\Vert ^{2(k+2i)}\mu({\rm d}v)\int_{S^{d-1}}g\left(\left\langle x,\hat{v}\right\rangle \hat{v}\right)\sigma({\rm d}\hat{v})\\
 & =\lambda_{d,k,i,\mu}g(x).
\end{align*}
The same argument applies for $d=1$, except we can consider $i=0$,
and that when $d=1$, $T_{\sigma}={\rm Id}$ (see Remark~\ref{rem:1d-telgarsky-ack}).
\end{proof}

\section{Transformations of polynomials }

We introduce transformations of polynomials to accommodate some shifting
and scaling, as well as changes in dimension.

\subsection{Linear transformations}\label{subsec:Linear-transformations}
\begin{defn}
Let $f\in\mathcal{P}_{d}$ be a polynomial, $A$ a $d'\times d$ matrix
and $x_{0}\in\mathbb{R}^{d}$. We say $\tilde{f}\in\mathcal{P}_{d'}$
is a $(A,x_{0})$-linearly transformed version of $f$ if 
\[
f(x)=\tilde{f}(A(x-x_{0})),\qquad x\in\mathbb{R}^{d}.
\]
\end{defn}

If $A$ is invertible then there exists such a linearly transformed
$f$, we may take $\tilde{f}=y\mapsto f(A^{-1}y+x_{0})$ If $A$ is
not invertible, e.g. if $A$ is $d'\times d$ with $d'<d$, then such
a function may not exist as it implies $f$ can be represented using
a lower dimensional linear transformation of $x$. If, on the other
hand, $d'>d$ then this allows us to embed $f$ within a higher-dimensional
polynomial. The following theorem provides a generalized ReLU integral
representation of $f$.
\begin{thm}
\label{thm:relu-transformed}Let $f\in\mathcal{P}_{d}$, $A$ a $d'\times d$
matrix and $x_{0}\in\mathbb{R}^{d}$. Let $\tilde{f}\in\mathcal{P}_{d'}$
be a $(A,x_{0})$-linearly transformed version of $f$, $\mu$ a spherically
symmetric probability measure on $\mathbb{R}^{d'}$, and $g=T_{\mu}^{-1}\tilde{f}$.
Then a generalized ReLU integral representation of $f$ is 
\[
f(x)=f(x_{0})+\int_{\mathbb{R}^{d'}}\int_{\mathbb{R}_{+}}u(v,s)\varsigma\left(\left\langle x,A^{T}v\right\rangle -\left\langle x_{0},A^{T}v\right\rangle -s\right)\varrho({\rm d}s)\mu({\rm d}v),
\]
where
\[
u(v,s)=\begin{cases}
2g_{v}'(0) & s=0,\\
2g_{v}''(s) & s>0.
\end{cases}
\]
If $A$ is invertible, then this corresponds to a standard ReLU integral
representation $(\mu_{A}\otimes\varrho,a,b,u_{A})$ where $\mu_{A}$
is the distribution of $A^{T}V$ when $V\sim\mu$, $a=f(x_{0})$,
$b(v,s)=-\left\langle x_{0},v\right\rangle -s$ and
\[
u_{A}(v,s)=u(A^{-T}v,s).
\]
\end{thm}

\begin{proof}
By Theorem~\ref{thm:relu-rep}, we obtain that $(\mu\otimes\varrho,\tilde{a},\tilde{b},u)$
is a ReLU integral representation of $\tilde{f}$ with $\tilde{a}=\tilde{f}(0)$
and $\tilde{b}(v,s)=-s$, i.e.
\[
\tilde{f}(y)=\tilde{f}(0)+\int_{\mathbb{R}^{d'}}\int_{\mathbb{R}_{+}}u(v,s)\varsigma\left(\left\langle y,v\right\rangle -s\right)\varrho({\rm d}s)\mu({\rm d}v),\quad y\in\mathbb{R}^{d'}.
\]
For any $x\in\mathbb{R}^{d},$
\[
f(x)=\tilde{f}(A(x-x_{0}))=\tilde{f}(0)+\int_{\mathbb{R}^{d'}}\int_{\mathbb{R}_{+}}u(v,s)\varsigma\left(\left\langle A(x-x_{0}),v\right\rangle -s\right)\varrho({\rm d}s)\mu({\rm d}v),
\]
from which we may conclude.
\end{proof}
\begin{rem}
\label{rem:mu-not-sph-sym-b-pos}We observe that $\mu_{A}$ need not
be spherically symmetric. Indeed, if $\mu=N(0,\sigma^{2}I_{d})$ then
$\mu_{A}=N(0,\sigma^{2}A^{T}A)$. We also see that $b(v,s)$ may take
positive values since it is possible that $\left\langle x_{0},v\right\rangle <-s$. 
\end{rem}

\subsection{Lifting by harmonic extension}

In Theorem~\ref{thm:general-soln-Tmu}, we observe that for spherically
symmetric $\mu$, the ReLU representative for a given $f$ is unique.
Moreover, the terms required to compensate for the action of $T_{\mu}$
are the eigenvalues $\lambda_{d,k,i,\mu}$ that are particularly sensitive
to $i>0$. In particular, if $\mu=N(0,\sigma^{2}I_{d})$ then Remark~\ref{rem:gaussian-C}
gives $\lambda_{d,k,i,\mu}$ which is dimension-independent if $i=0$
but $\mathcal{O}(d^{i})$ for $i>0$. So harmonic functions have representatives
that are not rescaled by amounts involving dimension $d$. However,
harmonic functions are a very small subset of the class of polynomials.

We consider a special case of the approach in Section~\ref{subsec:Linear-transformations},
where we embed $f$ in a higher-dimensional polynomial. In particular,
we exploit the simple fact that a polynomial $f\in\mathcal{P}_{d}$
can be extended to a polynomial $\mathcal{H}(f)\in\mathcal{H}_{d+1}$
that is harmonic and satisfies $\mathcal{H}(f)(\cdot,0)=f$. The following
construction is classical; generalizations of this kind of extension
exist such as \citet[Lemma~5.2]{garofalo2019structure}. The basic
idea is to introduce an additional variable $y\in\mathbb{R}$ and
define the function $\mathcal{H}(f)$ so that $\partial_{y}^{2}\mathcal{H}(f)=-\Delta_{x}\mathcal{H}(f)$
and hence $\Delta\mathcal{H}(f)=0$.
\begin{defn}[Harmonic extension]
For a polynomial $f\in\mathcal{P}_{d}$ we define $\mathcal{H}(f)\in\mathcal{P}_{d+1}$
to be the polynomial,
\[
\mathcal{H}(f)(x,y)=\sum_{j=0}^{\infty}\frac{(-1)^{j}y^{2j}}{(2j)!}\Delta^{j}f(x),\qquad x\in\mathbb{R}^{d},y\in\mathbb{R},
\]
which satisfies $\mathcal{H}(f)(x,0)=f(x)$ by construction.
\end{defn}

\begin{lem}
\label{lem:construct-harmonic-extension}If $f=\sum_{k\geq0}f_{k}$
is the decomposition of $f\in\mathcal{P}_{d}$ into homogeneous polynomials
$f_{k}\in\mathcal{P}_{d,k}$, then $\mathcal{H}(f)=\sum_{k\geq0}h_{k}$
is also such a decomposition for $\mathcal{H}(f)$ with 
\[
h_{k}(x,y)=\sum_{j=0}^{\left\lfloor k/2\right\rfloor }\frac{(-1)^{j}y^{2j}}{(2j)!}\Delta^{j}f_{k}(x),
\]
and each $h_{k}\in\mathcal{H}_{d+1,k}$. Hence, $\mathcal{H}(f)$
is harmonic.
\end{lem}

\begin{proof}
For $k\geq0$, $h_{k}$ is exactly the degree $k$ component of $\mathcal{H}(f)$.
We now show that $h_{k}$ is harmonic. We have
\[
\Delta h_{k}=\Delta_{x}h_{k}+\partial_{y}^{2}h_{k},
\]
and we find
\[
\partial_{y}^{2}\left[\frac{(-1)^{j}y^{2j}}{(2j)!}\Delta^{j}f_{k}(x)\right]=\frac{(-1)^{j}y^{2j-2}}{(2j-2)!}\Delta^{j}f_{k}(x)\mathbb{I}(j\geq1),
\]
while
\[
\Delta_{x}\left[\frac{(-1)^{j}y^{2j}}{(2j)!}\Delta^{j}f_{k}(x)\right]=\frac{(-1)^{j}y^{2j}}{(2j)!}\Delta^{j+1}f_{k}(x).
\]
It follows that 
\begin{align*}
\Delta h_{k}(x,y) & =\sum_{j=0}^{\left\lfloor k/2\right\rfloor }\frac{(-1)^{j}y^{2j}}{(2j)!}\Delta^{j+1}f_{k}(x)+\frac{(-1)^{j}y^{2j-2}}{(2j-2)!}\Delta^{j}f_{k}(x)\mathbb{I}(j\geq1)\\
 & =\sum_{j=0}^{\left\lfloor k/2\right\rfloor -1}\frac{(-1)^{j}y^{2j}}{(2j)!}\Delta^{j+1}f_{k}(x)+\sum_{j=1}^{\left\lfloor k/2\right\rfloor }\frac{(-1)^{j}y^{2j-2}}{(2j-2)!}\Delta^{j}f_{k}(x)\\
 & =\sum_{j=0}^{\left\lfloor k/2\right\rfloor -1}\frac{(-1)^{j}y^{2j}}{(2j)!}\Delta^{j+1}f_{k}(x)+\sum_{j=0}^{\left\lfloor k/2\right\rfloor -1}\frac{(-1)^{j+1}y^{2j}}{(2j)!}\Delta^{j+1}f_{k}(x)\\
 & =0,
\end{align*}
from which we may conclude.
\end{proof}
\begin{thm}
\label{thm:relu-rep-Hf}Let $f\in\mathcal{P}_{d}$ be a polynomial
and $\mu=N(0,\sigma^{2}I_{d+1})$. Then $(\mu\otimes\varrho,a,b,u)$
is a ReLU integral representation of $\mathcal{H}(f)=\sum_{k\geq0}h_{k}$,
where $h_{k}\in\mathcal{H}_{d+1,k}$ for $k\geq0$, $a=\mathcal{H}(f)(0)=f(0)$,
$b(v,s)=-s$ and
\[
u(v,s)=\begin{cases}
\frac{2}{\sigma^{2}}h_{1}(v) & s=0,\\
2\sum_{k\geq2}\frac{s^{k-2}}{(k-2)!\sigma^{2k}}h_{k}(v) & s>0.
\end{cases}
\]
\end{thm}

\begin{proof}
Let $\mathcal{H}(f)\in\mathcal{H}_{d+1}$ be the harmonic extension
of $f$. We may write $\mathcal{H}(f)=\sum_{k\geq0}h_{k}$ where each
$h_{k}\in\mathcal{H}_{d+1,k}$. Then Theorem~\ref{thm:general-soln-Tmu}
and Remark~\ref{rem:gaussian-C} imply
\[
g:=T_{\mu}^{-1}\mathcal{H}(f)=\sum_{k\geq0}\frac{1}{k!\sigma^{2k}}h_{k}.
\]
Using homogeneity, we have 
\[
g_{v}(s)=g(vs)=\sum_{k\geq0}\frac{1}{k!\sigma^{2k}}h_{k}(vs)=\sum_{k\geq0}\frac{s^{k}}{k!\sigma^{2k}}h_{k}(v),
\]
so $g_{v}'(s)=\sum_{k\geq1}\frac{s^{k-1}}{(k-1)!\sigma^{2k}}h_{k}(v)$,
$g_{v}'(0)=\frac{1}{\sigma^{2}}h_{1}(v)$ and $g_{v}''(s)=\sum_{k\geq2}\frac{s^{k-2}}{(k-2)!}h_{k}(v)$.
We conclude by applying Theorem~\ref{thm:relu-rep}.
\end{proof}

\subsection{Sharpened ReLU integral representation and the heat equation}

The main result of this section is Theorem~\ref{thm:f-relu-rep-integrated},
a \emph{sharpened} ReLU integral representation of $f\in\mathcal{P}_{d}$
with $\mu=N(0,\sigma^{2}I_{d})$. We use this term because sharpened
versions of the homogeneous components of $f$ arise naturally, corresponding
to running the heat equation on $f$ for negative time. The sharpening
can be thought of as a consequence of the harmonic extension of $f$
to $\mathcal{H}(f)$, followed by integration of the additional variable
in the resulting ReLU integral representation (Theorem~\ref{thm:relu-rep-Hf}).

The representation in Theorem~\ref{thm:relu-rep-Hf} is not associated
in general with the ReLU representative for $(\mu,f)$; indeed such
a representative is unique for spherically symmetric $\mu$ and given
by Theorem~\ref{thm:general-soln-Tmu}. Instead, we obtain the representation
for $f$ by first lifting $f$ to $\mathcal{H}(f)$ in $\mathbb{R}^{d+1}$,
obtaining the ReLU representative $g=T_{\mu}^{-1}\mathcal{H}(f)$
and then projecting the ReLU integral representation for $\mathcal{H}(f)$
back into $\mathbb{R}^{d}$. It is unclear if there is a more direct
route for obtaining this ReLU representation of $f$. One can think
of the representation for $\mathcal{H}(f)$ as incorporating a small
amount of noise around $0$ for the additional variable, which is
then integrated out to obtain the representation for $f$.
\begin{defn}[Solution operator for the heat equation]
\label{def:soln-op-heat}The operator 
\[
\exp\left(t\Delta\right)=\sum_{j=0}^{\infty}\frac{t^{j}}{j!}\Delta^{j},
\]
is the solution operator for the heat equation $\partial_{t}u=\Delta u$.
\end{defn}

To assist in interpretation, if $g=\exp\left(\frac{t}{2}\Delta\right)f$
for $t>0$ we have
\[
g(x)=\mathbb{E}\left[f(x+\sqrt{t}Z)\right],\qquad Z\sim N(0,I_{d}),
\]
and we can think of $g$ as a smoothed or noisy version of $f$ \citep[see, e.g.,][Section~9.2]{Durrett_2019}.
For polynomials, the inverse operator with $t<0$ is also well defined.
If $g=\exp\left(-\frac{t}{2}\Delta\right)f$, then
\[
f(x)=\mathbb{E}[g(x+\sqrt{t}Z)],\qquad Z\sim N(0,I_{d}),
\]
so $g$ is a sharpened or denoised version of $f$.
\begin{thm}
\label{thm:f-relu-rep-integrated}Let $f\in\mathcal{P}_{d}$, with
$f=\sum_{k\geq0}f_{k}$ its decomposition into homogeneous components,
$\mu=N(0,\sigma^{2}I_{d})$. Then $(\mu\otimes\varrho,a,b,u)$ is
a ReLU integral representation of $f$, where $a=f(0)$, $b(v,s)=-s$
and
\[
u(v,s)=\begin{cases}
\frac{2}{\sigma^{2}}f_{1}^{\sharp}(v) & s=0,\\
2\sum_{k\geq2}\frac{s^{k-2}}{(k-2)!\sigma^{2k}}f_{k}^{\sharp}(v) & s>0,
\end{cases}
\]
where $f_{k}^{\sharp}=\exp\left(-\frac{\sigma^{2}}{2}\Delta\right)f_{k}$
for $k\geq1$.
\end{thm}

\begin{proof}
Let $\mathcal{H}(f)=\sum_{k\geq0}h_{k}$ where $h_{k}\in\mathcal{H}_{d+1,k}$
for $k\geq0$. Using $f(x)=\mathcal{H}(f)(x,0),$ and $\left\langle (x,0),(v,w)\right\rangle =\left\langle x,v\right\rangle $,
and letting $\gamma=N(0,\sigma^{2})$, Theorem~\ref{thm:relu-rep-Hf}
gives the ReLU integral representation
\[
f(x)=f(0)+\int\tilde{u}((v,w),s)\varsigma\left(\left\langle x,v\right\rangle -s\right)\varrho({\rm d}s)\mu({\rm d}v)\gamma({\rm d}w),\quad x\in\mathbb{R}^{d},
\]
with 
\[
\tilde{u}((v,w),s)=\begin{cases}
\frac{2}{\sigma^{2}}h_{1}(v,w) & s=0,\\
2\sum_{k\geq2}\frac{s^{k-2}}{(k-2)!\sigma^{2k}}h_{k}(v,w) & s>0.
\end{cases}
\]
from which we may define 
\[
u(v,s)=\int\tilde{u}((v,w),s)\gamma({\rm d}w).
\]
Since $\int_{\mathbb{R}}w^{2j}\gamma({\rm d}w)=\sigma^{2j}(2j-1)!!=\sigma^{2j}\frac{(2j)!}{2^{j}j!},$
\begin{align*}
\int_{\mathbb{R}}h_{k}(v,w)\mu({\rm d}w) & =\int\sum_{j=0}^{\infty}\frac{(-1)^{j}w^{2j}}{(2j)!}\Delta^{j}f_{k}(v)\gamma({\rm d}w)\\
 & =\sum_{j=0}^{\infty}\frac{(-\sigma^{2})^{j}}{2^{j}j!}\Delta^{j}f_{k}(v)\\
 & =\exp\left(-\frac{\sigma^{2}}{2}\Delta\right)f_{k}(v)\\
 & =f_{k}^{\sharp}(v),
\end{align*}
we conclude.
\end{proof}
\begin{rem}
We consider the $s=0$ term. Observe that $f_{1}^{\sharp}=\exp\left(-\frac{\sigma^{2}}{2}\Delta\right)f_{1}=f_{1}$
since $\Delta f_{1}=0$. Hence,
\[
\int_{\mathbb{R}^{d}}u(v,0)\varsigma\left(\left\langle x,v\right\rangle \right)\mu({\rm d}v)=2\int_{\mathbb{R}^{d}}\frac{1}{\sigma^{2}}f_{1}(v)\varsigma\left(\left\langle x,v\right\rangle \right)\mu({\rm d}v).
\]
Notice that $f_{1}(v)=\left\langle c,v\right\rangle $ for some $c\in\mathbb{R}^{d}$,
and that $\nabla f(0)=\nabla f_{1}(0)=c$. Since
\[
\left\langle c,v\right\rangle \left\langle x,v\right\rangle =\left\langle c,-v\right\rangle \left\langle x,-v\right\rangle ,
\]
so by spherical symmetry of $\mu$, we have 
\begin{align*}
\int_{\mathbb{R}^{d}}u(v,0)\varsigma\left(\left\langle x,v\right\rangle \right)\mu({\rm d}v) & =2\int_{\mathbb{R}^{d}}\frac{1}{\sigma^{2}}\left\langle c,v\right\rangle \varsigma\left(\left\langle x,v\right\rangle \right)\mu({\rm d}v)\\
 & =\int_{\mathbb{R}^{d}}\frac{1}{\sigma^{2}}\left\langle c,v\right\rangle \left\langle x,v\right\rangle \mu({\rm d}v)\\
 & =\int_{\mathbb{R}^{d}}\frac{1}{\sigma^{2}}c^{T}vv^{T}x\mu({\rm d}v)\\
 & =\left\langle c,x\right\rangle .
\end{align*}
Hence, the $s=0$ term corresponds to the $\left\langle \nabla f(0),x\right\rangle $
term in the Taylor expansion of $f$ about $0$.
\end{rem}

It is natural to consider what $f_{k}^{\sharp}$ might look like.
For the case of polynomials, it is known that the result of applying
$\exp\left(-\frac{\sigma^{2}}{2}\Delta\right)$ to a polynomial $f$
is that the basis changes from monomials to generalized Hermite polynomials
without changing the coefficients. The relation between $\left\Vert f^{\sharp}\right\Vert _{L^{2}(\mu)}$
and $\left\Vert f\right\Vert _{F}\sigma^{k}$ is used for our quantitative
bounds in Section~\ref{sec:Quantitative-error-bounds}.
\begin{lem}
\label{lem:monom-to-herm-gen-sigma}If $f(x)=\sum_{\alpha}c_{\alpha}x^{\alpha}$
is the representation of a polynomial $f$ as a linear combination
of monomials, then $f^{\sharp}=\exp\left(-\frac{\sigma^{2}}{2}\Delta\right)f$
satisfies
\[
f^{\sharp}(x)=\sum_{\alpha}c_{\alpha}H_{\sigma,\alpha}(x),
\]
where $H_{\sigma,\alpha}$ is the (probabilist's) Hermite polynomial
with degree $\alpha$ for the measure $N(0,\sigma^{2}I),$
\[
H_{\sigma,\alpha}(x)=\sigma^{\left|\alpha\right|}H_{\alpha}\left(\frac{x}{\sigma}\right).
\]
Moreover, if $\mu=N(0,\sigma^{2}I)$ and $f$ is homogeneous of degree
$k$ 
\[
\left\Vert f^{\sharp}\right\Vert _{L^{2}(\mu)}=\left\Vert f\right\Vert _{F}\sigma^{k},
\]
where $\left\Vert f\right\Vert _{F}=\left(\sum_{\alpha}c_{\alpha}^{2}\alpha!\right)^{1/2}$
is the Fischer norm.
\end{lem}

\begin{proof}
Let $M_{k}=x\mapsto x^{k}$. For $d=1$, one may prove using generating
functions that
\[
\exp\left(-\frac{\sigma^{2}}{2}\Delta\right)M_{k}=H_{\sigma,k}.
\]
Then, since 
\[
\exp\left(-\frac{\sigma^{2}}{2}\Delta\right)=\exp\left(-\frac{1}{2}\sum_{j=1}^{d}\frac{\partial^{2}}{\partial x_{j}^{2}}\right)=\prod_{j=1}^{d}\exp\left(-\frac{1}{2}\frac{\partial^{2}}{\partial x_{j}^{2}}\right),
\]
we have 
\[
H_{\sigma,\alpha}(x)=\prod_{j=1}^{d}H_{\sigma,\alpha_{j}}(x_{j})=\prod_{j=1}^{d}\exp\left(-\frac{\sigma^{2}}{2}\frac{\partial^{2}}{\partial x_{j}^{2}}\right)M_{\alpha_{j}}(x_{j})=\exp\left(-\frac{\sigma^{2}}{2}\Delta\right)M_{\alpha}(x).
\]
Moreover, one can show using the expression for $H_{\sigma,\alpha}$
and the properties of standard Hermite polynomials,
\[
\left\langle H_{\sigma,\alpha},H_{\sigma,\beta}\right\rangle _{L^{2}(\mu)}=\mathbb{I}(\alpha=\beta)\alpha!\sigma^{2\left|\alpha\right|}.
\]
Hence,
\[
\left\Vert f^{\sharp}\right\Vert _{L^{2}(\mu)}^{2}=\sum_{\alpha}c_{\alpha}^{2}\alpha!\sigma^{2\left|\alpha\right|},
\]
from which we may conclude the last part, which assumes $f$ is homogeneous
of degree $k$.
\end{proof}
\begin{rem}
\label{rem:transform-then-extend-project}It is possible to combine
other transformations with the harmonic extension and projection.
For example, let $f\in\mathcal{P}_{d}$ be a polynomial and let $\tilde{f}\in\mathcal{P}_{d'}$
be an $(A,x_{0})$-linearly transformed version of $f$. We can obtain
by Theorem~\ref{thm:f-relu-rep-integrated} the integral representation,
involving the harmonic extension $\mathcal{H}(\tilde{f})$,
\[
\tilde{f}(z)=\tilde{f}(0)+\int_{\mathbb{R}^{d'}}\int_{\mathbb{R}_{+}}u(v,s)\varsigma\left(\left\langle z,v\right\rangle -s\right)\varrho({\rm d}s)\mu({\rm d}v),\qquad z\in\mathbb{R}^{d'},
\]
which by Theorem~\ref{thm:relu-transformed} provides the representation,
\[
f(x)=f(x_{0})+\int_{\mathbb{R}^{d'}}\int_{\mathbb{R}_{+}}u(v,s)\varsigma\left(\left\langle x,A^{T}v\right\rangle +b(v,s)\right)\varrho({\rm d}s)\mu({\rm d}v),\qquad x\in\mathbb{R}^{d},
\]
where $b(v,s)=-\left\langle x_{0},A^{T}v\right\rangle -s$.
\end{rem}

\begin{example}
Let $f\in\mathcal{P}_{1}$ be $f(x)=x^{2}$. The harmonic extension
is $h(x,y)=x^{2}-y^{2}$. For $\mu=N(0,1)$, we have $f^{\sharp}(x)=x^{2}-1$.
So the sharpened integral representation is
\[
f(x)=2\int_{\mathbb{R}}\int_{\mathbb{R}_{+}}(v^{2}-1)\max\left(xv-s,0\right){\rm d}s\,N(v;0,1){\rm d}v.
\]
In contrast, the direct representation without a harmonic extension
is obtained by $g=T_{\mu}^{-1}f=x\mapsto x^{2}/3$, so that $g_{v}''(s)=\frac{2}{3}v^{2}$,
from which we obtain
\[
f(x)=2\int_{\mathbb{R}}\int_{\mathbb{R}_{+}}\frac{2}{3}v^{2}\max\left(xv-s,0\right){\rm d}s\,N(v;0,1){\rm d}v.
\]
If we instead take a direct approach with $\mu={\rm Uniform}\left(\left\{ -1,1\right\} \right)$,
we obtain $g=f$, so that $g_{v}''(s)=2v^{2}$, from which we obtain
\[
f(x)=2\sum_{v\in\{-1,1\}}\int_{\mathbb{R}_{+}}2v^{2}\max\left(xv-s,0\right){\rm d}s.
\]
In an experiment, the trained network appears to be more consistent
with the harmonic extension approach, i.e. the $u_{i}$ take small
negative as well as a broader range of positive values, and the bias
terms are almost exclusively negative. Of course, this could also
be due to the choice of network initialization.
\end{example}

\section{Quantitative error bounds}\label{sec:Quantitative-error-bounds}

We focus here on the ReLU integral representation of $f$ obtained
in Theorem~\ref{thm:f-relu-rep-integrated}. More precisely, on the
optimized version of that representation after the change of measure
to minimize the expected $L^{2}(\mathcal{D})$ error presented in
Proposition~\ref{prop:network-mse-bound}. The main result is Theorem~\ref{thm:quant-bound-deterministic},
which considers the $L^{2}(\mathcal{D})$ error of a two-layer network.
We assume that $\mathcal{D}$ is sub-Gaussian to obtain quantitative
bounds.
\begin{defn}
We say $X$ is sub-Gaussian with mean $0$ and variance proxy $\sigma_{X}^{2}$
if 
\[
\mathbb{P}\left(\left\langle X,u\right\rangle \geq t\right)\leq\exp\left(-\frac{t^{2}}{2\sigma_{X}^{2}}\right),
\]
for any unit vector $u$.
\end{defn}

\begin{example}
Here are some examples of sub-Gaussian variance proxies:
\begin{enumerate}
\item If $X\sim\mathcal{N}(0,\sigma^{2}I_{d})$ then $\sigma_{X}^{2}=\sigma^{2}$.
\item If $X_{i}$ are independent, mean $0$ and take values in $[a,b]$
then $\sigma_{X}^{2}=(b-a)^{2}/4$. E.g. if $X\sim{\rm Uniform}([-a,a]^{d})$
then $\sigma_{X}^{2}=a^{2}$.
\item If $X$ is uniform in the ball or on the sphere of radius $r$ in
$\mathbb{R}^{d}$ then $\sigma_{X}^{2}=r^{2}/d$.
\end{enumerate}
In particular, for a parameter $\sigma_{X}^{2}$ that scales as $1/d$
we may take $X\sim N(0,\frac{1}{d}I_{d})$, $X\sim{\rm Uniform}(S^{d-1})$
or $X\sim{\rm Uniform}(B_{1})$ where $S^{d-1}$ is the unit sphere
centred at $0$ and $B_{1}$ the unit ball in $\mathbb{R}^{d}$ centred
at $0$. 

\end{example}

\begin{thm}
\label{thm:quant-bound-deterministic}Assume $\mathcal{D}$ is sub-Gaussian
with mean $0$ and variance proxy $\sigma_{X}^{2}$, and that $f\in\mathcal{P}_{d}$
with $f=a+\sum_{k\geq1}f_{k}$, $f_{k}\in\mathcal{P}_{d,k}$. Then
there exists a two-layer network $f_{n}$ such that
\[
\left\Vert f_{n}-f\right\Vert _{L^{2}(\mathcal{D})}\leq\frac{1}{\sqrt{n}}\left\{ \left(\sum_{k\geq1}C_{d,k}\left\Vert f_{k}\right\Vert _{F}\right)^{2}-\left\Vert f-a\right\Vert _{L^{2}(\mathcal{D})}^{2}\right\} ^{1/2}\leq\frac{1}{\sqrt{n}}\sum_{k\geq1}C_{d,k}\left\Vert f_{k}\right\Vert _{F},
\]
where $C_{d,1}=2^{3/2}\sigma_{X}d^{1/2}$, and for $k\geq2$,
\[
C_{d,k}=2^{\frac{3k-1}{2}}\sigma_{X}^{k}\frac{\Gamma\left(\frac{k-1}{2}\right)}{(k-2)!}\frac{\Gamma(k+\frac{d}{2})^{1/2}}{\Gamma(\frac{d}{2})^{1/2}}\leq\sigma_{X}^{k}\sqrt{2k}\left(2e\cdot\frac{d+k-1}{k}\right)^{k/2}.
\]
\end{thm}

\begin{proof}
This follows from Proposition~\ref{prop:network-mse-bound} and Proposition~\ref{prop:quant-bound-muphi}.
\end{proof}
We observe that the bound has no dependence on $\sigma^{2}$, the
variance parameter of $\mu$, which is consistent with Remark~\ref{rem:sphere-invariance-rep}:
the choice of $\sigma^{2}$ is irrelevant. In order that the bound
$C_{d,k}\left\Vert f_{k}\right\Vert _{F}$ is controlled as $d\to\infty$,
it is necessary that $\left\Vert f_{k}\right\Vert _{F}\in\mathcal{O}((\sigma_{X}\sqrt{d})^{-k})$.
In order that $\sum_{k}C_{d,k}\left\Vert f_{k}\right\Vert _{F}<\infty$,
it is necessary that $\left\Vert f_{k}\right\Vert _{F}\in o\left(\left(\sigma_{X}\sqrt{2e}\right)^{-k}k^{-3/2}\right)$
as $k\to\infty$. If the condition $\left\Vert f_{k}\right\Vert _{F}\in\mathcal{O}((\sigma_{X}d)^{-k/2})$
holds then $\left\Vert f_{k}\right\Vert _{F}\in o\left(\left(\sigma_{X}\sqrt{2e}\right)^{-k}k^{-3/2}\right)$
will also hold when $d\geq6$ since $6>2e$. For large $d$, one can
see that if $\sigma_{X}=d^{-1/2}$, then the $C_{d,k}$ are dimension-independent,
suggesting that suitably regular functions can be approximated well
on the surface of the unit sphere or in the unit ball irrespective
of dimension and degree.

If we assume $\sigma_{X}=1$, e.g. if $\mathcal{D}=N(0,I_{d})$, then
an example of a function that achieves $\left\Vert f_{k}\right\Vert _{F}\in\mathcal{O}((\sigma_{X}\sqrt{d})^{-k})$
is $f(x)=g(d^{-1/2}x)$, where $g=\sum g_{k}$ and $\left\Vert g_{k}\right\Vert _{F}\in\mathcal{O}(1)$
as $d\to\infty$. Such an $f$ can be viewed as not giving any particular
component of $x$ too much importance. The conditions on Fischer norms
can of course be viewed as conditions on the partial derivatives of
$f$ at $0$, in light of the discussion in Section~\ref{sec:Introduction}.
Another implication is that if $f=\varphi(d^{-1/2}\sum_{i=1}^{d}x_{i})$,
where $\varphi(u)=\sum_{k\geq0}c_{k}u^{k}$ then one may deduce that
$\left\Vert f_{k}\right\Vert _{F}=\left\Vert \varphi_{k}\right\Vert _{F}$,
where $f_{k}$ and $\varphi_{k}$ are the degree $k$ homogeneous
components of $f$ and $\varphi$, respectively.
\begin{lem}
\label{lem:sub-gaussian-norm}Assume $\mathcal{D}$ is sub-Gaussian
with mean $0$ and variance proxy $\sigma_{X}^{2}$. Then
\[
\left\Vert \varsigma\left(\left\langle \cdot,v\right\rangle -s\right)\right\Vert _{L^{2}(\mathcal{D})}\leq\sqrt{2}\sigma_{X}\left\Vert v\right\Vert \exp\left(-\frac{s^{2}}{4\sigma_{X}^{2}\left\Vert v\right\Vert ^{2}}\right).
\]
Moreover,
\[
\int_{\mathbb{R}_{+}}\left\Vert \varsigma\left(\left\langle \cdot,v\right\rangle -s\right)\right\Vert _{L^{2}(\mathcal{D})}s^{k-2}{\rm d}s\leq2^{k-\frac{3}{2}}\left(\sigma_{X}\left\Vert v\right\Vert \right)^{k}\Gamma\left(\frac{k-1}{2}\right).
\]
\end{lem}

\begin{proof}
Since $X$ is sub-Gaussian with mean $0$, we may deduce that $\left\langle X,v\right\rangle $
is sub-Gaussian with variance proxy $\delta^{2}=\sigma_{X}^{2}\left\Vert v\right\Vert ^{2}$.
\[
\mathbb{P}_{\mathcal{D}}\left(\left\langle X,v\right\rangle \geq t\right)\leq\exp\left(-\frac{t^{2}}{2\delta^{2}}\right).
\]
We obtain
\begin{align*}
\left\Vert \varsigma\left(\left\langle \cdot,v\right\rangle -s\right)\right\Vert _{L^{2}(\mathcal{D})}^{2} & =\mathbb{E}\left[\varsigma\left(\left\langle X,v\right\rangle -s\right){}^{2}\right]\\
 & =\int_{0}^{\infty}\mathbb{P}\left(\varsigma\left(\left\langle X,v\right\rangle -s\right)^{2}>t\right){\rm d}t\\
 & =\int_{0}^{\infty}\mathbb{P}\left(\left\langle X,v\right\rangle -s>\sqrt{t}\right){\rm d}t\\
 & =\int_{0}^{\infty}\mathbb{P}\left(\left\langle X,v\right\rangle >\sqrt{t}+s\right){\rm d}t\\
 & \leq\int_{0}^{\infty}\exp\left(-\frac{(\sqrt{t}+s)^{2}}{2\delta^{2}}\right){\rm d}t\\
 & =2\delta^{2}\exp\left(-\frac{s^{2}}{2\delta^{2}}\right)-2\sqrt{2\pi}s\delta\Phi\left(-\frac{s}{\delta}\right)\\
 & \leq2\delta^{2}\exp\left(-\frac{s^{2}}{2\delta^{2}}\right).
\end{align*}
The last statement follows by integrating the square root of this
bound multiplied by $s^{k-2}$.
\end{proof}
\begin{prop}
\label{prop:quant-bound-muphi}Assume $\mathcal{D}$ is sub-Gaussian
with mean $0$ and variance proxy $\sigma_{X}^{2}$, and that $f=\sum_{k\geq0}f_{k}$
is the decomposition of $f$ into homogeneous components. Let $\mu=N(0,\sigma^{2}I_{d})$
and $(\mu\otimes\varrho,a,b,u)$ be the ReLU integral representation
from Theorem~\ref{thm:f-relu-rep-integrated}. Then with $\phi$
as in Proposition~\ref{prop:network-mse-bound}, we obtain
\[
\int_{\mathbb{R}^{d}}\int_{\mathbb{R}_{+}}\left\Vert \phi(\cdot;v,s)\right\Vert _{L^{2}(\mathcal{D})}\rho({\rm d}s)\mu({\rm d}v)\leq\sum_{k\geq1}C_{d,k}\left\Vert f_{k}\right\Vert _{F},
\]
where $C_{d,k}$ are defined in Theorem~\ref{thm:quant-bound-deterministic}.
\end{prop}

\begin{proof}
We have $\phi(x;v,s)=\varsigma\left(\left\langle x,v\right\rangle -s\right)u(v,s)$,
so
\[
\int_{\mathbb{R}^{d}}\int_{\mathbb{R}_{+}}\left\Vert \phi(\cdot;v,s)\right\Vert _{L^{2}(\mathcal{D})}\varrho({\rm d}s)\mu({\rm d}v)=\int_{\mathbb{R}^{d}}\int_{\mathbb{R}_{+}}\left\Vert \varsigma\left(\left\langle \cdot,v\right\rangle -s\right)\right\Vert _{L^{2}(\mathcal{D})}\left|u(v,s)\right|\varrho({\rm d}s)\mu({\rm d}v).
\]
First, consider $s=0$. Then observe that $f_{1}^{\sharp}=f_{1}$
since $\Delta^{j}f_{1}=0$ for any $j\geq1$, so 
\[
\left|u(v,0)\right|=\frac{2}{\sigma^{2}}\left|f_{1}(v)\right|.
\]
By Lemma~\ref{lem:sub-gaussian-norm}, Cauchy--Schwarz and Lemma~\ref{lem:monom-to-herm-gen-sigma}
\begin{align*}
\int_{\mathbb{R}^{d}}\left\Vert \phi(\cdot;v,0)\right\Vert _{L^{2}(\mathcal{D})}\mu({\rm d}v) & \leq\frac{2\sqrt{2}}{\sigma^{2}}\sigma_{X}\int\left\Vert v\right\Vert \left|f_{1}(v)\right|\mu({\rm d}v)\\
 & \leq\frac{2\sqrt{2}}{\sigma^{2}}\sigma_{X}\left\{ \int\left\Vert v\right\Vert ^{2}\mu({\rm d}v)\right\} ^{1/2}\left\Vert f_{1}\right\Vert _{L^{2}(\mu)}\\
 & =\frac{2\sqrt{2}}{\sigma}\sigma_{X}\left\{ 2\sigma^{2}\frac{\Gamma(\frac{d}{2}+1)}{\Gamma(\frac{d}{2})}\right\} ^{1/2}\left\Vert f_{1}\right\Vert _{L^{2}(\mu)}\\
 & =2\sqrt{2}\sigma_{X}d^{1/2}\left\Vert f_{1}\right\Vert _{F}.
\end{align*}
Now consider $s>0$, so 
\[
\left|u(v,s)\right|=\left|2\sum_{k\geq2}\frac{s^{k-2}}{(k-2)!\sigma^{2k}}f_{k}^{\sharp}(v)\right|,
\]
where $f_{k}^{\sharp}=\exp\left(-\frac{\sigma^{2}}{2}\Delta\right)f_{k}$.
Applying Lemma~\ref{lem:sub-gaussian-norm}, Cauchy--Schwarz and
Lemma~\ref{lem:monom-to-herm-gen-sigma},
\begin{align*}
 & \int_{\mathbb{R}^{d}}\int_{\mathbb{R}_{+}}\left\Vert \phi(\cdot;v,s)\right\Vert _{L^{2}(\mathcal{D})}{\rm d}s\mu({\rm d}v)\\
= & \int_{\mathbb{R}^{d}}\int_{\mathbb{R}_{+}}\left\Vert \varsigma\left(\left\langle \cdot,v\right\rangle -s\right)\right\Vert _{L^{2}(\mathcal{D})}\left|2\sum_{k\geq2}\frac{s^{k-2}}{(k-2)!\sigma^{2k}}f_{k}^{\sharp}(v)\right|{\rm d}s\mu({\rm d}v)\\
\leq & 2\int_{\mathbb{R}^{d}}\sum_{k\geq2}\int_{\mathbb{R}_{+}}\left\Vert \varsigma\left(\left\langle \cdot,v\right\rangle -s\right)\right\Vert _{L^{2}(\mathcal{D})}\frac{s^{k-2}}{(k-2)!\sigma^{2k}}\left|f_{k}^{\sharp}(v)\right|{\rm d}s\mu({\rm d}v)\\
\leq & 2\int_{\mathbb{R}^{d}}\sum_{k\geq2}2^{k-\frac{3}{2}}\left(\sigma_{X}\left\Vert v\right\Vert \right)^{k}\Gamma\left(\frac{k-1}{2}\right)\frac{1}{(k-2)!\sigma^{2k}}\left|f_{k}^{\sharp}(v)\right|\mu({\rm d}v)\\
= & \sum_{k\geq2}2^{k-\frac{1}{2}}\frac{\Gamma\left(\frac{k-1}{2}\right)}{(k-2)!\sigma^{2k}}\sigma_{X}^{k}\int_{\mathbb{R}^{d}}\left\Vert v\right\Vert ^{k}\left|f_{k}^{\sharp}(v)\right|\mu({\rm d}v)\\
\leq & \sum_{k\geq2}2^{k-\frac{1}{2}}\frac{\Gamma\left(\frac{k-1}{2}\right)}{(k-2)!\sigma^{2k}}\sigma_{X}^{k}\left\{ \int_{\mathbb{R}^{d}}\left\Vert v\right\Vert ^{2k}\mu({\rm d}v)\right\} ^{1/2}\left\Vert f_{k}^{\sharp}\right\Vert _{L^{2}(\mu)}\\
= & \sum_{k\geq2}2^{k-\frac{1}{2}}\frac{\Gamma\left(\frac{k-1}{2}\right)}{(k-2)!\sigma^{2k}}\sigma_{X}^{k}\sigma^{k}2^{\frac{k}{2}}\frac{\Gamma(k+\frac{d}{2})^{1/2}}{\Gamma(\frac{d}{2})^{1/2}}\left\Vert f_{k}\right\Vert _{F}\sigma^{k}\\
= & \sum_{k\geq2}2^{\frac{3k-1}{2}}\sigma_{X}^{k}\frac{\Gamma\left(\frac{k-1}{2}\right)}{(k-2)!}\frac{\Gamma(k+\frac{d}{2})^{1/2}}{\Gamma(\frac{d}{2})^{1/2}}\left\Vert f_{k}\right\Vert _{F}.
\end{align*}
The expression for $C_{d,k}$ follows. For the upper bound, observe
that using the arithmetic-geometric mean inequality for
\[
\frac{\Gamma(k+\frac{d}{2})}{\Gamma(\frac{d}{2})}=\left(\frac{d}{2}\right)\left(\frac{d}{2}+1\right)\cdots\left(\frac{d}{2}+k-1\right),
\]
we obtain
\[
\frac{\Gamma(k+\frac{d}{2})^{1/2}}{\Gamma(\frac{d}{2})^{1/2}}\leq\left(\frac{d+k-1}{2}\right)^{k/2}.
\]
Using the Legendre duplication formula, we have
\[
\frac{\Gamma\left(\frac{k-1}{2}\right)}{(k-2)!}=\frac{\Gamma\left(\frac{k-1}{2}\right)}{\Gamma(k-1)}=\frac{2^{2-k}\sqrt{\pi}}{\Gamma(\frac{k}{2})}.
\]
Then applying the bound $\Gamma(x)>\sqrt{2\pi}x^{x-1/2}\exp(-x)$,
we find
\[
\Gamma\left(\frac{k}{2}\right)>\sqrt{2\pi}\left(\frac{k}{2}\right)^{(k-1)/2}\exp\left(-\frac{k}{2}\right),
\]
and the bound results by combining the terms.
\end{proof}
The numerical values of the bounds are likely to be somewhat conservative.
The sub-Gaussian bound in Lemma~\ref{lem:sub-gaussian-norm} is not
equally accurate for all $s>0$, and hence the integral bound may
also be quite conservative, even if the dependence on $\sigma_{X}$,
$\left\Vert v\right\Vert $ and $k$ are appropriate. The strategy
in the proof of Proposition~\ref{prop:quant-bound-muphi} also involves
use of the triangle inequality to allow integration with respect to
$\varrho$, which may not be ideal.

\section{Observations from trained two-layer ReLU networks}

We perform a few experiments to determine if the integral representations
are consistent with trained networks. There is no attempt for these
experiments to be conclusive about any specific hypothesis. A more
sophisticated analysis as well as, potentially, further theory would
be required to do this in a compelling way, as there are two main
obstacles. First, it is unknown which linear transformations (see
Section~\ref{subsec:Linear-transformations}) might be performed.
Second, although the unoptimized ReLU integral representation $(\mu\otimes\varrho,a,b,u)$
is fairly simple, the optimized representation $(\pi^{\star},a,b,u_{\pi^{\star}})$
is not straightforward to determine as it involves reweighting the
simple normal distribution $\mu$ and introducing a distribution on
$s$ that takes into account $\left\Vert \phi(\cdot;v,s)\right\Vert _{L^{2}(\mathcal{D})}$,
which seems not to be very tractable in general. Hence, it is not
obvious concretely what one should compare a trained network to.

It is quite common to initialize ReLU networks with the $v_{i}$ being
$N(0,\frac{\sigma^{2}}{d})$ and $u$ being $N(0,\frac{\sigma^{2}}{n})$
for some $\sigma^{2}$. For example, \citet{he2015delving} is a popular
approach suggesting to use $\sigma^{2}=2$. We use this initialization
for our experiments. In both examples considered below, the $L^{2}(\mathcal{D})$
error of the individual neurons $\phi(\cdot;v_{i},s_{i})$ are similar
across neurons, which is consistent with optimized representations,
see Remark~\ref{rem:l2-error-neurons}.

One of the features of the representations viewed so far is that if
there is no linear transformation with a shift term $x_{0}$, as in
Section~\ref{subsec:Linear-transformations}, then the bias terms
should all be non-positive. We consider a simple model where
\[
f(x)=\left\Vert x-c\right\Vert ^{2},
\]
and $\mathcal{D}=N(c,d^{-1/2}I_{d})$. We observe predominantly negative
biases with $c=0$ but a mixture of negative and positive biases when
$c=3d^{-1/2}\cdot{\bf 1}$. In addition, the second layer bias term
$a$ is close to $0=f(c)$ in both cases, suggesting that the model
has been shifted to $\tilde{f}(z)=\left\Vert z\right\Vert ^{2}$ with
$f(x)=\tilde{f}(x-c)$. The empirical distribution of the $v_{i}$
is very close to the initialization distribution in both cases, which
is unsurprising given the simplicity and symmetry of $\tilde{f}$.

We also examine an example with different dependence on the components
of $x$:
\[
f(x)=x_{1}^{2}+x_{2}^{2}+x_{3}^{2}+x_{1}x_{2}x_{3}+10x_{1}^{4}+10x_{2}^{4},
\]
which has a lack of symmetry in that the coefficient of $x_{3}^{4}$
is $0$. In this case, we see that the covariance $\Sigma$ of the
$v_{i}$ in the trained network tends to be diagonal with $\Sigma_{33}$
significantly smaller than $\Sigma_{11}\approx\Sigma_{22}$, which
is consistent with $A$ being diagonal with $A_{1}=A_{2}>A_{3}$,
and so $\tilde{f}(z)=f(A^{-1}z)$ scales $x_{1}$ and $x_{2}$ down
to significantly lower the Fischer norm of the degree $4$ component
while increasing the Fischer norms of the degree $2$ and $3$ components.

\section{An RKHS perspective and a different representation}\label{sec:An-RKHS-perspective}

The quantitative bounds involving the Fischer norm suggest that an
appropriate reproducing kernel Hilbert space (RKHS) for functions
approximated by a two-layer neural network is the RKHS $\mathcal{F}$
with reproducing kernel $K(x,y)=\exp\left(\left\langle x,y\right\rangle \right)$
since this RKHS $\mathcal{F}$ has exactly the Fischer norm as its
norm, i.e. $\left\Vert f\right\Vert _{\mathcal{F}}=\left\Vert f\right\Vert _{F}$.
This RKHS is closely related to the Segal--Bargmann space of complex
functions \citep[see, e.g.,][]{paulsen2016introduction}, which has
kernel $K_{{\rm SB}}(x,y)=\exp\left(\left\langle x,\bar{y}\right\rangle \right)$.
It is also related to the RKHS associated with the kernel $K_{{\rm G}}(x,y)=\exp\left(-\frac{1}{2}\left\Vert x-y\right\Vert ^{2}\right)$
since $K(x,y)=G(x)K_{{\rm G}}(x,y)G(y)$ with $G(z)=\exp\left(-\frac{1}{2}\left\Vert z\right\Vert ^{2}\right)=K_{{\rm G}}(z,0)$.
From the norm, we have that $\mathcal{F}$ contains functions that
are analytic at $0$ with Taylor expansions
\[
f(x)=\sum_{\alpha}\frac{\partial^{\alpha}f(0)}{\alpha!}x^{\alpha},\qquad x\in\mathbb{R}^{d},
\]
satisfying $\sum_{\alpha\in\mathbb{N}_{0}^{d}}\partial^{\alpha}f(0)^{2}/\alpha!<\infty$.
This is somewhat different to the kernels and spaces that have been
studied in the past: \citet{bach2017breaking} provides some interesting
examples.

It seems that the integral representation in Theorem~\ref{thm:f-relu-rep-integrated}
is related to the kernel $K$, but we do not pursue that further here.
Instead, we find that a simple argument involving $K$ gives a cleaner
and closely related integral representation with a better bound described
only in terms of $\left\Vert f\right\Vert _{\mathcal{F}}$, rather
than the larger $\sum_{k\geq0}\left\Vert f_{k}\right\Vert _{\mathcal{F}}$.
The Gaussian function $G$ mentioned above appears, which is fixed
and independent of $f$. In high dimensions, with sub-Gaussian $\mathcal{D}$
satisfying $\sigma_{X}=d^{-1/2}$, one will often have $\left\Vert X\right\Vert $
concentrated near $1$, so that $G(X)$ is nearly constant. Nevertheless,
the infinite neuron two-layer network does not approximate $f$ exactly
unless $\mathcal{D}$ is support on a sphere centered at the origin.
A direct proof involves only two fairly simple observations, but we
provide some comments afterwards that connect the approach to the
RKHS $\mathcal{F}$ more explicitly, for those interested. We also
show that the same derivation works for a wide variety of activations
satisfying a very simple condition. In this section we define $\gamma=N(0,I_{d})$.
Unlike the previous representations in this paper, the representation
is not explicitly based on integrating along lines and a Taylor expansion,
i.e. Theorem~\ref{thm:relu-rep}.
\begin{thm}
\label{thm:Gqrep-exp}Let $f\in\mathcal{F}$. Then $(\gamma\otimes{\rm Leb}(\mathbb{R}),0,b,u)$
is a ReLU integral representation of $q$ satisfying $f=G\cdot q$,
where $b(v,s)=-s$ and $u(v,s)=f^{\#}(v)\exp(s)$, where $f^{\#}=\exp\left(-\frac{1}{2}\Delta\right)f$.
\end{thm}

\begin{proof}
The existence of $f^{\#}\in L^{2}(\gamma)$ is assured by observing
that if $f(x)=\sum_{\alpha}c_{\alpha}x^{\alpha}$ and $\left\Vert f\right\Vert _{\mathcal{F}}<\infty$
then one may write $f^{\#}=\sum_{\alpha}c_{\alpha}H_{\alpha}$ and
$\left\Vert f^{\#}\right\Vert _{L^{2}(\gamma)}=\left\Vert f\right\Vert _{\mathcal{F}}<\infty$.
First observe that $f^{\#}$ satisfies $f(x)=\mathbb{E}\left[f^{\#}(Y)\right]$
for $Y\sim N(x,I_{d})$. Hence, we have
\begin{align*}
f(x) & =\int_{\mathbb{R}^{d}}f^{\#}(y)(2\pi)^{-\frac{d}{2}}\exp\left(-\frac{1}{2}\left\langle y-x,y-x\right\rangle \right){\rm d}y\\
 & =\int_{\mathbb{R}^{d}}f^{\#}(y)(2\pi)^{-\frac{d}{2}}\exp\left(\left\langle x,y\right\rangle \right)\exp\left(-\frac{1}{2}\left\Vert x\right\Vert ^{2}\right)\exp\left(-\frac{1}{2}\left\Vert y\right\Vert ^{2}\right){\rm d}y\\
 & =G(x)\int_{\mathbb{R}^{d}}f^{\#}(y)\exp\left(\left\langle x,y\right\rangle \right)\gamma({\rm d}y).
\end{align*}
Now, observe that for $z\in\mathbb{R}$,
\[
\int_{\mathbb{R}}\varsigma(z-s)\exp(s){\rm d}s=\int_{-\infty}^{z}(z-s)\exp(s){\rm d}s=\exp(z),
\]
from which we may conclude that
\[
f(x)=G(x)\int_{\mathbb{R}^{d}}\int_{\mathbb{R}}\exp(s)f^{\#}(v)\varsigma\left(\left\langle x,v\right\rangle -s\right){\rm d}s\gamma({\rm d}v).
\]
\end{proof}
\begin{rem}
The same representation can be adjusted for any activation function
$\varsigma$ such that 
\[
0<\int_{\mathbb{R}}\varsigma(u)\exp(-u){\rm d}u<\infty.
\]
Indeed, letting $\varrho(s)=\exp(s)/\int_{\mathbb{R}}\varsigma(u)\exp(-u){\rm d}u$,
we may deduce that
\[
\int_{\mathbb{R}}\varsigma(z-s)\varrho(s){\rm d}s=\frac{\int_{\mathbb{R}}\varsigma(z-s)\exp(s){\rm d}u}{\int_{\mathbb{R}}\varsigma(u)\exp(-u){\rm d}u}=\frac{\int_{\mathbb{R}}\varsigma(u)\exp(z-u){\rm d}u}{\int_{\mathbb{R}}\varsigma(u)\exp(-u){\rm d}u}=\exp(z).
\]
This condition does hold if $\varsigma$ is a polynomial, or for the
sigmoid function.
\end{rem}

\begin{rem}
The result follows from elementary observations, combined with the
smoothing identity $f(x)=\mathbb{E}_{\gamma}\left[f^{\#}(x+Z)\right]$.
It can alternatively be viewed using the reproducing property of $K$
in $\mathcal{F}$. Indeed, $\exp\left(-\frac{1}{2}\Delta\right)$
is an isometric isomorphism from $\mathcal{F}$ to $L^{2}(\gamma)$,
so 
\[
\left\langle f,g\right\rangle _{\mathcal{F}}=\left\langle \exp\left(-\frac{1}{2}\Delta\right)f,\exp\left(-\frac{1}{2}\Delta\right)g\right\rangle _{L^{2}(\gamma)}.
\]
A related isometric isomorphism for the Segal--Bargmann space is
of interest in physics. Together with the reproducing property of
$K$ in $\mathcal{F}$, we obtain
\begin{align*}
f(x) & =\left\langle f,K(x,\cdot)\right\rangle _{\mathcal{F}}\\
 & =\left\langle \exp\left(-\frac{1}{2}\Delta\right)f,\exp\left(-\frac{1}{2}\Delta\right)K(x,\cdot)\right\rangle _{L^{2}(\gamma)}\\
 & =\left\langle f^{\#},\exp\left(\left\langle x,\cdot\right\rangle -\frac{1}{2}\left\Vert x\right\Vert ^{2}\right)\right\rangle _{L^{2}(\gamma)},
\end{align*}
which gives the same identify for $f$.
\end{rem}

We now present the $L^{2}(\mathcal{D})$ error bound, which involves
optimizing the two-layer network through the choice the probability
measure $\pi\gg\gamma\otimes{\rm Leb}(\mathbb{R})$.
\begin{thm}
Let $\mathcal{D}$ be a distribution, which is sub-Gaussian with mean
$0$ and variance proxy $\sigma_{X}^{2}<1/2$. Let $f\in\mathcal{F}$.
Then there exists a two-layer network $q_{n}$ such that $f_{n}=G\cdot q_{n}$
satisfies
\[
\left\Vert f-f_{n}\right\Vert _{L^{2}(\mathcal{D})}\leq\frac{1}{\sqrt{n}}\left\{ C(d,\sigma_{X})^{2}\left\Vert f\right\Vert _{\mathcal{F}}^{2}-\left\Vert f\right\Vert _{L^{2}(\mathcal{D})}^{2}\right\} \leq\frac{1}{\sqrt{n}}C(d,\sigma_{X})\left\Vert f\right\Vert _{\mathcal{F}},
\]
where 
\[
C(d,\sigma_{X})=2\sqrt{2\pi d(d+2)}\frac{\sigma_{X}^{2}\left\{ 2\sigma_{X}^{2}+1\right\} ^{d/4+1/2}}{\left\{ 1-2\sigma_{X}^{2}\right\} ^{d/4+1}}.
\]
If $d>3$ and $\sigma_{X}=d^{-1/2}$, then 
\[
C(d,\sigma_{X})\leq2\exp(1)\sqrt{2\pi}\cdot\exp\left(\frac{6}{d-2}\right)\left\Vert f\right\Vert _{\mathcal{F}}.
\]
\end{thm}

\begin{proof}
We consider the representation from Theorem~\ref{thm:Gqrep-exp},
defining
\[
\phi(x;v,s)=G(x)u(v,s)\varsigma\left(\left\langle x,v\right\rangle -s\right).
\]
By Proposition~\ref{prop:network-mse-bound}, it suffices to upper
bound
\[
\int_{\mathbb{R}^{d}}\int_{\mathbb{R}}\left|u(v,s)\right|\left\Vert \phi(\cdot;v,s)\right\Vert _{L^{2}(\mathcal{D})}{\rm d}s\gamma({\rm d}v).
\]
Since $\left\Vert x\right\Vert \left\Vert v\right\Vert \geq\left|\left\langle x,v\right\rangle \right|$
by Cauchy--Schwarz, we have $\left\Vert x\right\Vert >s/\left\Vert v\right\Vert $
if $\left\langle x,v\right\rangle >s$. Hence,
\begin{align*}
\int_{\mathbb{R}^{d}}G(x)^{2}\varsigma\left(\left\langle x,v\right\rangle -s\right)^{2}\mathcal{D}({\rm d}x) & =\int_{\mathbb{R}^{d}}\exp\left(-\left\Vert x\right\Vert ^{2}\right)\varsigma\left(\left\langle x,v\right\rangle -s\right)^{2}{\bf 1}_{\left\langle x,v\right\rangle >s}\mathcal{D}({\rm d}x)\\
 & \leq\exp\left(-\frac{s^{2}}{\left\Vert v\right\Vert ^{2}}\right)\left\Vert \varsigma\left(\left\langle \cdot,v\right\rangle -s\right)\right\Vert _{L^{2}(\mathcal{D})}^{2},
\end{align*}
from which we obtain the bound, using Lemma \ref{lem:sub-gaussian-norm},
\[
\left\Vert G(\cdot)\varsigma\left(\left\langle \cdot,v\right\rangle -s\right)\right\Vert _{L^{2}(\nu)}\leq\sqrt{2}\sigma_{X}\left\Vert v\right\Vert \exp\left(-\frac{s^{2}}{2\left\Vert v\right\Vert ^{2}}\left\{ 1+\frac{1}{2\sigma_{X}^{2}}\right\} \right).
\]
Now, by exact integration and Cauchy--Schwarz,
\begin{align*}
 & \int_{\mathbb{R}^{d}}\int_{\mathbb{R}}\left|u(v,s)\right|\left\Vert \phi(\cdot;v,s)\right\Vert _{L^{2}(\mathcal{D})}{\rm d}s\gamma({\rm d}v)\\
\leq & \int_{\mathbb{R}^{d}}\left|f^{\#}(v)\right|\int_{\mathbb{R}}\exp(s)\sqrt{2}\sigma_{X}\left\Vert v\right\Vert \exp\left(-\frac{s^{2}}{2\left\Vert v\right\Vert ^{2}}\left\{ 1+\frac{1}{2\sigma_{X}^{2}}\right\} \right){\rm d}s\gamma({\rm d}v)\\
= & 2\sqrt{2\pi}\int_{\mathbb{R}^{d}}\left|f^{\#}(v)\right|\frac{\sigma_{X}^{2}\left\Vert v\right\Vert ^{2}}{\left\{ 2\sigma_{X}^{2}+1\right\} ^{1/2}}\exp\left\{ \frac{\left\Vert v\right\Vert ^{2}\sigma_{X}^{2}}{2\sigma_{X}^{2}+1}\right\} \gamma({\rm d}v)\\
\leq & 2\sqrt{2\pi}\left\Vert f^{\#}\right\Vert _{L^{2}(\mu)}\frac{\sigma_{X}^{2}}{\left\{ 2\sigma_{X}^{2}+1\right\} ^{1/2}}\left|\int_{\mathbb{R}^{d}}\left\Vert v\right\Vert ^{4}\exp\left\{ 2\frac{\left\Vert v\right\Vert ^{2}\sigma_{X}^{2}}{2\sigma_{X}^{2}+1}\right\} \gamma({\rm d}v)\right|^{1/2}\\
= & 2\sqrt{2\pi}\left\Vert f^{\#}\right\Vert _{L^{2}(\mu)}\frac{\sigma_{X}^{2}}{\left\{ 2\sigma_{X}^{2}+1\right\} ^{1/2}}\left|d(d+2)\cdot\left\{ \frac{2\sigma_{X}^{2}+1}{1-2\sigma_{X}^{2}}\right\} ^{d/2+2}\right|^{1/2}\\
= & 2\sqrt{2\pi d(d+2)}\frac{\sigma_{X}^{2}\left\{ 2\sigma_{X}^{2}+1\right\} ^{d/4+1/2}}{\left\{ 1-2\sigma_{X}^{2}\right\} ^{d/4+1}}\left\Vert f\right\Vert _{\mathcal{F}}.
\end{align*}
If $d\geq3$ and $\sigma_{X}=1/d$, then we have
\begin{align*}
2\sqrt{2\pi d(d+2)}\frac{\sigma_{X}^{2}\left\{ 2\sigma_{X}^{2}+1\right\} ^{d/4+1/2}}{\left\{ 1-2\sigma_{X}^{2}\right\} ^{d/4+1}} & =2\sqrt{2\pi}\left(\frac{1+\frac{2}{d}}{1-\frac{2}{d}}\right)^{\frac{d}{4}+1}\left\Vert f\right\Vert _{\mathcal{F}}\\
 & \leq2\exp(1)\sqrt{2\pi}\exp\left(\frac{6}{d-2}\right)\left\Vert f\right\Vert _{\mathcal{F}}.
\end{align*}
\end{proof}
\begin{rem}
Of course, if one departs from ReLU networks and considers a two-layer
network using the exponential kernel and post-multiplication by $G(x)$,
one could use the representation
\[
f=\int_{\mathbb{R}^{d}}f^{\#}(v)G(\cdot)\exp\left(\left\langle \cdot,v\right\rangle \right)\gamma({\rm d}v).
\]
For simplicity, one can derive the corresponding quantitative bound
when $\mathcal{D}=N(0,\sigma_{X}^{2}I_{d})$ and $\sigma_{X}^{2}<1/2$,
using Cauchy--Schwarz,
\begin{align*}
 & \int_{\mathbb{R}^{d}}\left|f^{\#}(v)\right|\left\Vert G\cdot\exp\left(\left\langle \cdot,v\right\rangle \right)\right\Vert _{L^{2}(\mathcal{D})}\gamma({\rm d}v)\\
= & \int_{\mathbb{R}^{d}}\left|f^{\#}(v)\right|\left\{ \frac{1}{1+2\sigma_{X}^{2}}\right\} ^{d/4}\exp\left(\frac{\sigma_{X}^{2}\left\Vert v\right\Vert ^{2}}{1+2\sigma_{X}^{2}}\right)\gamma({\rm d}v)\\
\leq & \left\Vert f^{\#}\right\Vert _{L^{2}(\gamma)}\left\{ \frac{1}{1+2\sigma_{X}^{2}}\right\} ^{d/4}\left\{ \frac{1+2\sigma_{X}^{2}}{1-2\sigma_{X}^{2}}\right\} ^{d/4}\\
= & \left\Vert f\right\Vert _{\mathcal{F}}\left(1-2\sigma_{X}^{2}\right)^{-d/4}.
\end{align*}
If $d\geq3$ and $\sigma_{X}=d^{-1/2}$, this means that there exists
a two-layer exponential network of this type with $n$ neurons such
that 
\[
\left\Vert f-f_{n}\right\Vert _{L^{2}(\mathcal{D})}\leq\frac{1}{\sqrt{n}}\exp\left(\frac{1}{2}\right)\exp\left(\frac{1}{2(d-2)}\right)\left\Vert f\right\Vert _{\mathcal{F}}.
\]
We see that the approximation of the exponential kernel using a sum
of ReLU functions does not contribute very much to the error bound,
i.e. it could be that the number of neurons is determined primarily
by the need to average over several directions $v$.
\end{rem}

We close this Section by observing that one seemingly important distinction
between this work and some other kernel-based approaches, is that
we treat the exponential kernel $K$ as an integral of ReLU functions,
but define the RKHS in terms of $K$ itself. It is plausible that
this is an appropriate space of functions for two-layer networks.
An alternative, natural approach is to consider the operator $L_{\pi}:L^{2}(\pi)\to L^{2}(\mathcal{D})$,
\[
L_{\pi}u(x)=\int_{\mathbb{R}^{d}\times\mathbb{R}_{+}}u(v,s)\varsigma\left(\left\langle v,x\right\rangle -s\right)\pi({\rm d}v,{\rm d}s),
\]
which we find does indeed define a kernel via 
\[
L_{\pi}L_{\pi}^{*}f(x)=\int_{\mathbb{R}^{d}}\left\{ \int_{\mathbb{R}^{d}\times\mathbb{R}_{+}}\varsigma\left(\left\langle v,y\right\rangle -s\right)\varsigma\left(\left\langle v,x\right\rangle -s\right)\pi({\rm d}v,{\rm d}s)\right\} f(y)\mathcal{D}({\rm d}y).
\]
From Theorem~\ref{thm:f-relu-rep-integrated}, we know that for suitable
$\pi$, there exist $u$ such that $L_{\pi}u=f$ for any $f\in\mathcal{P}_{d}$.
On the other hand, it can be shown that $u\not\in L^{2}(\pi)$, due
to the dependence on $s$, so that it is not clear if the functions
$f$ are actually in the RKHS with reproducing kernel
\[
(x,y)\mapsto\int_{\mathbb{R}^{d}\times\mathbb{R}_{+}}\varsigma\left(\left\langle v,y\right\rangle -s\right)\varsigma\left(\left\langle v,x\right\rangle -s\right)\pi({\rm d}v,{\rm d}s).
\]
We note that the kernel above is related to the family of arc-cosine
kernels \citep{cho2009kernel}.

In contrast, if we decompose $T_{K}=L_{K}L_{K}^{*}$, where $T_{K}g(x)=\int K(x,y)f(y)\gamma({\rm d}y)$,
then it is straightforward to obtain that $L_{K}:\ell^{2}\to L^{2}(\gamma)$
maps coefficients of normalized monomials to the corresponding polynomials,
while $L_{K}^{*}:L^{2}(\gamma)\to\ell^{2}$ maps polynomials to coefficients
by integration against each normalized monomial, i.e.
\[
L_{K}c(x)=\sum_{\alpha}c(\alpha)\frac{x^{\alpha}}{\sqrt{\alpha!}},\qquad L_{K}^{*}g(\alpha)=\int_{\mathbb{R}^{d}}\frac{y^{\alpha}}{\sqrt{\alpha!}}g(y)\gamma({\rm d}y),
\]
and here we clearly see that the summability of $\sum_{\alpha}c(\alpha)^{2}=\sum_{\alpha}\left(\frac{c(\alpha)}{\sqrt{\alpha!}}\right)^{2}\alpha!=\left\Vert L_{K}c\right\Vert _{\mathcal{F}}$,
so that $c\in\ell^{2}$ implies $L_{K}c\in\mathcal{F}$.

\section{Discussion}\label{sec:Discussion}

We have established some explicit ReLU integral representations for
polynomials, which seem appealing at least from a mathematical perspective.
The representations, as well as the bounds for the \emph{sharpened}
ReLU integral representation, are consistent with the invariance of
the network to the magnitude of the vectors $v_{i}$. The representations
do not induce spherically symmetric distributions on $v$, after transformations
and error optimization are performed. There appears to be some agreement
with trained two-layer networks, but this is far from conclusive.
The representations and the quantitative bounds may be useful for
future research in a number of directions.

We have also presented a connection to the RKHS with reproducing kernel
$K(x,y)=\exp\left(\left\langle x,y\right\rangle \right)$, with a
two-layer representation that involves post-multiplication by a fixed,
Gaussian function. The activation function in this case can be from
a whole class of activations satisfying a certain integral constraint.
There are clearly a large number of integral representations that
can be pursued, and it would be of interest to understand which representations
are most appropriate.

A natural question is how the approach could be used to analyze individual
layers within a deeper ReLU network: it seems plausible that this
could be fruitful, although it is of course unclear what the intermediate
functions could be. Nevertheless, the functions in each layer could
be considerably more complicated than in carefully developed sparse
networks \citep[see, e.g.,][]{schmidt_hieber}, which do not seem
to capture the behaviour of trained deep networks. For example, a
natural direction to consider is to extend $T_{\mu}$ to consider
vector-valued polynomials, since such functions are just stacks of
real-valued polynomials, then for a function $f=f_{L}\circ\cdots f_{1}$,
we could set $g_{\ell}=T_{\mu}^{-1}f_{\ell}$ for each $\ell\in\{1,\ldots,L\}$,
to obtain
\[
f=\left(T_{\mu}g_{L}\right)\circ\cdots\circ\left(T_{\mu}g_{1}\right),
\]
noting that each $g_{\ell}$ corresponds to a family of integral representations.

In another direction, there could be other important actions on the
function $f$ to consider, or different mechanisms to produce different
integral representations. The connections between the representations
here and different representations in prior work \citep[see, e.g.,][]{sonoda2017neural,savarese2019infinite,petrosyan2020neural}
are unclear. It is possible that different operators could lead to
different integral representations that are superior in some way,
or accommodate a wider range of distributions $\mu$. It is also natural
to investigate how non-polynomial functions should be represented
as polynomials or approximated via such networks.

Finally, there has been recent progress on viewing the training of
two-layer networks as an optimization on the space of distributions,
motivated by integral representations of such networks \citet{chizat2018global,mei2019mean}:
it is conceivable that novel ReLU integral representations such as
this are helpful in understanding how such representations are connected
to, or inform, the statistical aspects of network training, which
have not been addressed here at all.

\section*{Acknowledgements}

This research was supported by the Engineering and Physical Sciences
Research Council (Prob\_AI, EP/Y028783/1). I would like to thank Christophe
Andrieu for thoughtful comments on an early draft.

\bibliographystyle{abbrvnat}
\bibliography{relu-approx}

\appendix

\section{Optimal squared $L^{2}$ error for random networks}\label{sec:Optimal-squared-}

In this section we establish the minimal expected squared $L^{2}(\mathcal{D})$
error of an IID average of random functions, when $f=\int\phi(\cdot,z)\nu({\rm d}z)$
for some measure $\nu$ and the optimization is performed with respect
to the choice of distribution of $z$. The argument is a variation
on a classical argument for identifying the optimal importance distribution
in simple Monte Carlo integration.
\begin{lem}
\label{lem:optimal-random-fcn-dist}Let $f(\cdot)=\int_{\mathsf{Z}}\phi(\cdot,z)\nu({\rm d}z)$,
where $\nu$ is a measure. Let 
\[
f_{n}(\cdot)=\frac{1}{n}\sum_{i=1}^{n}w(Z_{i})\phi(\cdot;Z_{i}),
\]
where $Z=(Z_{1},\ldots,Z_{n})$ is a random vector of independent
$\pi$-distributed random variables with $\pi\gg\nu$ and ${\rm d}\nu/{\rm d}\pi=w$.
Let $\mathcal{D}$ be a probability measure. Then 
\begin{align*}
\mathbb{E}_{\pi}\left[\left\Vert f_{n}-f\right\Vert _{L^{2}(\mathcal{D})}^{2}\right] & =\frac{1}{n}\mathbb{E}_{\pi}\left[\left\Vert w(Z_{1})\phi(\cdot;Z_{1})-f\right\Vert _{L^{2}(\mathcal{D})}^{2}\right]\\
 & \geq\frac{1}{n}\left\{ \left[\int_{\mathsf{Z}}\nu({\rm d}z)\left\Vert \phi(\cdot;z)\right\Vert _{L^{2}(\mathcal{D})}\right]^{2}-\left\Vert f\right\Vert _{L^{2}(\mathcal{D})}^{2}\right\} ,
\end{align*}
and the lower bound is attained for $\pi=\pi^{\star}$, where
\[
\pi^{\star}({\rm d}z)=\frac{\nu({\rm d}z)\left\Vert \phi(\cdot;z)\right\Vert _{L^{2}(\mathcal{D})}}{\int_{\mathsf{Z}}\nu({\rm d}z')\left\Vert \phi(\cdot;z')\right\Vert _{L^{2}(\mathcal{D})}},
\]
i.e. 
\[
w^{\star}(z)=\frac{{\rm d}\nu}{{\rm d}\pi^{\star}}=\frac{1}{\left\Vert \phi(\cdot;z)\right\Vert _{L^{2}(\mathcal{D})}}\int_{\mathsf{Z}}\nu({\rm d}z')\left\Vert \phi(\cdot;z')\right\Vert _{L^{2}(\mathcal{D})}.
\]
\end{lem}

\begin{proof}
Observe that
\[
f(x)=\int\phi(x,z)\nu({\rm d}z)=\int\phi(x,z)w(z)\pi({\rm d}z).
\]
Hence, we have $\mathbb{E}\left[f_{n}(x)\right]=f(x)$ for arbitrary
$x$. Let $\phi_{w}(x;z)=w(z)\phi(x;z)$, and write $f_{n}(\cdot;z)=\frac{1}{n}\sum_{i=1}^{n}\phi_{w}(\cdot;z_{i})$.
\begin{align*}
\mathbb{E}_{\pi}\left[\left\Vert f_{n}(\cdot;Z)-f\right\Vert _{L^{2}(\mathcal{D})}^{2}\right] & =\int_{\mathsf{Z}^{n}}\int_{\mathbb{R}^{d}}\left|f_{n}(x;z)-f(x)\right|^{2}\mathcal{D}({\rm d}x)\pi^{\otimes n}({\rm d}z)\\
 & =\frac{1}{n}\int_{\mathbb{R}^{d}}{\rm var}_{\pi}\left(\phi_{w}(x;Z_{1})\right)\mathcal{D}({\rm d}x)\\
 & =\frac{1}{n}\int_{\mathbb{R}^{d}}\mathbb{E}_{\pi}\left[\phi_{w}(x;Z_{1})^{2}\right]-f(x)^{2}\mathcal{D}({\rm d}x)\\
 & =\frac{1}{n}\left\{ \mathbb{E}_{\pi}\left[\left\Vert \phi_{w}(\cdot;Z_{1})\right\Vert _{L^{2}(\mathcal{D})}^{2}\right]-\left\Vert f\right\Vert _{L^{2}(\mathcal{D})}^{2}\right\} .
\end{align*}
Jensen's inequality gives
\begin{align*}
\mathbb{E}_{\pi}\left[\left\Vert \phi_{w}(\cdot;Z_{1})\right\Vert _{L^{2}(\mathcal{D})}^{2}\right] & =\mathbb{E}_{\pi}\left[w(Z_{1})^{2}\left\Vert \phi(\cdot;Z_{1})\right\Vert _{L^{2}(\mathcal{D})}^{2}\right]\\
 & \geq\mathbb{E}_{\pi}\left[w(Z_{1})\left\Vert \phi(\cdot;Z_{1})\right\Vert _{L^{2}(\mathcal{D})}\right]^{2}\\
 & =\left\{ \int_{\mathsf{Z}}\nu({\rm d}z)\left\Vert \phi(\cdot;z)\right\Vert _{L^{2}(\mathcal{D})}\right\} ^{2},
\end{align*}
while 
\begin{align*}
\mathbb{E}_{\pi^{\star}}\left[\left\Vert \phi_{w^{\star}}(\cdot;Z_{1})\right\Vert _{L^{2}(\mathcal{D})}^{2}\right] & =\mathbb{E}_{\pi^{\star}}\left[w^{\star}(Z_{1})^{2}\left\Vert \phi(\cdot;Z_{1})\right\Vert _{L^{2}(\mathcal{D})}^{2}\right]\\
 & =\left\{ \int_{\mathsf{Z}}\nu({\rm d}z)\left\Vert \phi(\cdot;z)\right\Vert _{L^{2}(\mathcal{D})}\right\} ^{2},
\end{align*}
from which we may conclude.
\end{proof}
\begin{rem}
\label{rem:l2-error-neurons}In the proof, we see that the optimal
distribution $\pi^{\star}$ is such that $\phi_{w^{\star}}$ has $\left\Vert \phi_{w^{\star}}(\cdot;z)\right\Vert _{L^{2}(\mathcal{D})}=\int_{\mathsf{Z}}\nu({\rm d}z')\left\Vert \phi(\cdot;z')\right\Vert _{L^{2}(\mathcal{D})}$
for $\pi^{\star}$ almost all $z$. This suggests that for a trained
network, each of the functions $\phi(\cdot;z)$ in the ensemble should
have a similar $L^{2}(\mathcal{D})$ norm.
\end{rem}

\section{Funk--Hecke formulas}\label{sec:Funk=002013Hecke-formulas}

The following formula identifies eigenvalue-eigenfunction pairs for
the operator $S_{\varphi}$ where 
\[
S_{\varphi}g(u)=\int_{S^{d-1}}\varphi\left(\left\langle u,v\right\rangle \right)g(v)\sigma({\rm d}v),\qquad u\in S^{d-1},
\]
with $\sigma={\rm Uniform}(S^{d-1})$, which is a generalization of
$S_{k,i}$ defined in Section~\ref{subsec:Spectral-analysis-of}.
\begin{lem}[{\citet[Theorem~2.22]{atkinson2012spherical}}]
\label{lem:phi-mvp}Let $\varphi:[-1,1]\to\mathbb{R}$ be such that
$\int_{-1}^{1}\varphi(t)(1-t^{2})^{\frac{d-3}{2}}{\rm d}t$ is finite.
Then with $h_{k}\in\mathcal{H}_{d,k}$,
\[
S_{\varphi}h_{k}=\lambda_{d,k,\varphi}h_{k},
\]
where
\[
\lambda_{d,k,\varphi}=\frac{\left|S^{d-2}\right|}{\left|S^{d-1}\right|}\int_{-1}^{1}\varphi(t)P_{d,k}(t)(1-t^{2})^{\frac{d-3}{2}}{\rm d}t,
\]
and $P_{d,k}$ is the restriction of the Legendre polynomial of degree
$k$ in $d$ dimensions on the unit sphere,
\[
P_{d,k}(t)=n!\Gamma\left(\frac{d-1}{2}\right)\sum_{i=0}^{[k/2]}(-1)^{i}\frac{(1-t^{2})^{i}t^{k-2i}}{4^{i}i!(k-2i)!\Gamma(i+\frac{d-1}{2})}.
\]
\end{lem}

\begin{proof}[Proof of Lemma~\ref{lem:explicit-lambda-k}]
Lemma~\ref{lem:phi-mvp} applies with $\varphi(t)=t^{m}$, where
$m=2k+i$. Using \citet[Proposition~2.26]{atkinson2012spherical},
we compute
\begin{align*}
\lambda_{d,k,i} & =\frac{\left|S^{d-2}\right|}{\left|S^{d-1}\right|}\int_{-1}^{1}\varphi(t)P_{d,k}(t)(1-t^{2})^{\frac{d-3}{2}}{\rm d}t\\
 & =\frac{\Gamma(\frac{d}{2})}{\sqrt{\pi}\Gamma(\frac{d-1}{2})}\cdot\frac{\Gamma(\frac{d-1}{2})}{2^{k}\Gamma(k+\frac{d-1}{2})}\int_{-1}^{1}\varphi^{(k)}(t)(1-t^{2})^{k+\frac{d-3}{2}}{\rm d}t\\
 & =\frac{\Gamma(\frac{d}{2})}{\sqrt{\pi}}\cdot\frac{1}{2^{k}\Gamma(k+\frac{d-1}{2})}\cdot\frac{m!}{(m-k)!}\cdot\int_{-1}^{1}t^{m-k}(1-t^{2})^{k+\frac{d-3}{2}}{\rm d}t\\
 & =\frac{\Gamma(\frac{d}{2})}{\sqrt{\pi}2^{k}\Gamma(k+\frac{d-1}{2})}\cdot\frac{m!}{(m-k)!}\cdot\frac{\Gamma(\frac{m-k+1}{2})\Gamma(\frac{2k+d-1}{2})}{\Gamma(\frac{m+k+d}{2})}\\
 & =\frac{\Gamma(\frac{d}{2})}{\sqrt{\pi}2^{k}}\cdot\frac{m!}{(m-k)!}\cdot\frac{\Gamma(\frac{m-k+1}{2})}{\Gamma(\frac{m+k+d}{2})}\\
 & =\frac{\Gamma(\frac{d}{2})}{\sqrt{\pi}2^{k}}\cdot\frac{(k+2i)!}{(2i)!}\cdot\frac{\Gamma(i+\frac{1}{2})}{\Gamma(k+i+\frac{d}{2})}\\
 & =\frac{\Gamma(\frac{d}{2})}{2^{2i+k}i!}\cdot\frac{(k+2i)!}{\Gamma(k+i+\frac{d}{2})}.
\end{align*}
\end{proof}

\end{document}